\documentclass[twoside]{article}

\usepackage[accepted]{aistats2023}
\usepackage[round]{natbib}

\usepackage[utf8]{inputenc} 
\usepackage[T1]{fontenc}    
\usepackage{hyperref}       
\usepackage{url}            
\usepackage{booktabs}       
\usepackage{amsfonts}       
\usepackage{nicefrac}       
\usepackage{microtype}      
\usepackage{xcolor}         
\usepackage{svg}

\usepackage{amsthm}
\usepackage{enumitem}
\usepackage{defs}
\usepackage[ruled,noend]{algorithm2e}

\newcommand{\papertitle}{
A Blessing of Dimensionality
in
Membership Inference through Regularization}

\begin{document}

%

%

\twocolumn[

\aistatstitle{\papertitle}

\aistatsauthor{Jasper Tan \And Daniel LeJeune \And Blake Mason \And Hamid Javadi \And Richard G.\ Baraniuk}

\aistatsaddress{ Rice University \And Stanford University \And Rice University \And Rice University \And Rice University} ]

\begin{abstract}
Is overparameterization a privacy liability?
In this work, we study the effect that the number of parameters has on a classifier's vulnerability to membership inference attacks.
We first demonstrate how the number of parameters of a model can induce a privacy--utility trade-off: increasing the number of parameters generally improves generalization performance at the expense of lower privacy.
However, remarkably, we then show that if coupled with proper regularization, increasing the number of parameters of a model can actually simultaneously increase \textit{both} its privacy and performance, thereby eliminating the privacy--utility trade-off.
Theoretically, we demonstrate this curious phenomenon for logistic regression with ridge regularization in a bi-level feature ensemble setting.
Pursuant to our theoretical exploration, we develop a novel leave-one-out analysis tool to precisely characterize the vulnerability of a linear classifier to the optimal membership inference attack.
We empirically exhibit this ``blessing of dimensionality'' for neural networks on a variety of tasks using early stopping as the regularizer.
\end{abstract}



\section{INTRODUCTION}\label{sec:intro}

Recently, the machine learning community has been gravitating towards the trend of increasingly overparameterized models, which have been shown both theoretically \citep{belkin2020two, hastie2019surprises, mei2022generalization} and empirically \citep{kaplan2020scaling,nakkiran2021deep} to generalize better than their smaller counterparts in diverse settings. These findings encourage machine learning system designers to opt for the largest possible model to maximize performance on unseen data.

However, when training machine learning models on sensitive data \citep{chen2019gmail, batmaz2019review, googlespeech}, it is also crucial to understand the attendant privacy issues to prevent data leaks.
Alarmingly, multiple attacks have been developed in the literature to perform {\em membership inference} (MI), which extracts information about specific examples in a model's training dataset, even when given only black-box access \citep{ fredrikson2015model,shokri2017membership}.

Is the trend of increasingly overparameterizing models detrimental to privacy? In this paper, we focus on the effect that the number of parameters of a model has on its vulnerability to MI attacks. We study this problem both theoretically and empirically.

We first demonstrate a parameter-wise privacy--utility trade-off: increasing the number of parameters of a model increases its generalization performance while also increasing its vulnerability to MI attacks. We show this theoretically for logistic regression and empirically for neural networks (NNs). This corroborates previous empirical \citep{carlini2021extracting, mireshghallah2022quantifying} and theoretical \citep{tan2022parameters} findings that larger models are less private.

However, we then show that this is not the end of the story between overparameterization and privacy. Remarkably, we discover that if proper \textit{regularization} is incorporated while increasing the number of parameters, the larger model can actually enjoy greater privacy (stronger protection from MI attacks) for the same generalization performance as its smaller counterpart. That is, there is a ``blessing of dimensionality,'' rather than a curse, and more overparameterized models can in some cases in fact be more private when paired with regularization. We show this behavior theoretically for logistic regression with ridge regularization and empirically for neural networks with early stopping.

This behavior is due to the fact that regularization induces its own privacy--utility trade-off: beyond a point, increasing regularization provide greater protection from MI attacks while decreasing generalization performance. However, the trade-off induced by regularization for a larger network traces a trajectory of lower MI vulnerability and better generalization performance than the trade-off for a smaller network. That is, larger networks have better regularization-wise privacy--utility trade-offs.

To demonstrate this effect theoretically, 
we must be able to precisely characterize the output distribution of a model on a fixed training data point over the randomness of all other training data.
We overcome this challenge by developing a novel leave-one-out analysis tool based on the convex Gaussian min-max theorem \citep{thrampoulidis2018cgmt,salehi_logistic_2019} that we apply to high-dimensional logistic regression in the asymptotic regime. 
We believe our theoretical tool may be of independent interest to other researchers pursuing theoretical studies of privacy for machine learning models.
In this work, we use our tool to provide a precise asymptotic characterization of MI for the optimal black-box MI attack. 

For the practitioner, our analysis encourages considering the number of parameters when designing privacy-preserving machine learning models. In particular, larger models may be more beneficial if they are carefully coupled with proper regularization. In summary, our paper has three core contributions:
\begin{enumerate}
    \item 
    We demonstrate how individually increasing either the number of parameters or decreasing the regularization of a classification model can \textbf{decrease its privacy.}
    
    \item 
    We discover multiple situations where wider NNs enjoy an improved regularization-induced privacy--utility trade-off compared to narrow ones, and that, controlling for the privacy level by regularization, \textbf{increased generalization performance due to overparameterization is not at odds with privacy.}
    
    \item We theoretically analyze high-dimensional logistic regression in the asymptotic regime and replicate our empirical NN observations for a bi-level feature ensemble using a \textbf{novel leave-one-out analysis that may be of independent interest.} Using this tool, we also derive the fundamental MI vulnerability for overparameterized logistic regression models.
\end{enumerate}

\paragraph{Related Work.}

This work contributes to the rapidly growing field of
membership inference (MI), a framework being increasingly used to study the privacy implications of machine learning models.
Previous works have shown how MI is in principle a task of hypothesis testing with the optimal adversary being the likelihood ratio test (LRT) \citep{sablayrolles2019white, carlini2021membership}.
We leverage this optimal LRT adversary in our theoretical analysis.
Since the distributions for the LRT are typically not known for general models such as neural networks, more practical attack strategies such as binary classification \citep{shokri2017membership, salem2018ml} and perturbation-based inference \citep{pmlr-v139-choquette-choo21a, kaya2020effectiveness} have been proposed.
We refer the reader to \citet{hu2021membership} for a comprehensive survey of MI attacks.
For our neural network experiments, we use the loss thresholding attack introduced by \citet{yeom2018} and improved by \citet{ye2021enhanced} due to its simplicity and effectiveness.

Prior work has also studied how various types of regularization affect MI attacks \citep{song2019membership, wang2020against, kaya2021does, galinkin2021influence, rezaei2021accuracy}.
There are limited studies on the effect of overparameterization on MI.
\citet{tan2022parameters} analyze how linear regression models are more susceptible to MI as they become more overparameterized, and \citet{carlini2021extracting,mireshghallah2022quantifying} empirically observe larger language models being more vulnerable to MI than their smaller counterparts.
\citet{yeom2018} study the theoretical connection between overfitting and membership advantage but do not connect this to (over)parameterization.

In addition to MI,
differential privacy (DP) is another popular framework used to study the privacy implications of machine learning algorithms \citep{dwork2008differential, abadi2016deep, ha2019differential}. 
Differentially private training algorithms ensure that models obtained when training on datasets differing in one data point do not differ much.
\citet{yu2021differentially,li2022largelanguage} show that larger models achieve better utility for the same DP amount when using fine-tuning, echoing our message that larger models can have better privacy--utility trade-offs than smaller ones.
By providing rigorous worst-case guarantees, DP also protects models from MI attacks \citep{yeom2018}, but typically at the cost of having very low utility \citep{rahman2018membership, jayaraman2019evaluating, cai2021costofprivacy}. Indeed, it has been shown that DP techniques provide poorer MI defense vs.\ utility trade-offs than other MI defense schemes \citep{kaya2020effectiveness, liu2021generalization}. 
Furthermore, while they provide powerful information-theoretic guarantees, it is not clear how the DP metrics of $(\epsilon, \delta)$ translate to vulnerability from real-world MI attacks.
As such, we believe both MI and DP analyses complement each other in providing a comprehensive understanding of privacy-preserving machine learning, and we focus on MI in this work.

\begin{figure*}[ht]
	\centering
	\includegraphics[width=\textwidth]{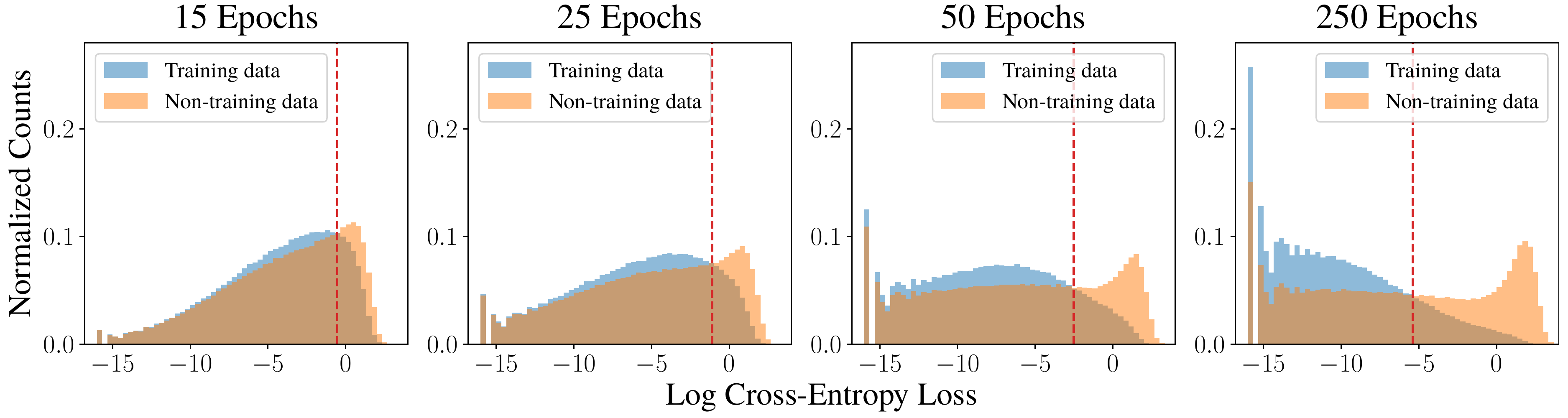}
	\caption{\textbf{Loss gap increases with epochs.} Empirical histograms of $\log$ cross-entropy losses for training and non-training data points of 20 ResNet18s ($w=64$) trained on CIFAR10 for different training epochs. While both distributions generally shift towards smaller losses with more epochs, losses for training points shift more quickly than those for non-training points, enabling loss-threshold attacks. The optimal threshold is depicted with a red dashed line. 
	This illustrates why loss threshold MI accuracy increases with epochs (Figure \ref{fig:epochwise}). For visualization purposes, we drop points that achieve 0 (to machine precision) loss.}
	\label{fig:losses}
\end{figure*}

Our work is strongly related to the ``double descent`` literature that studies the relationship of overparameterization and generalization error \citep{dar2021farewell}.
\citet{nakkiran2021deep} demonstrate double descent behavior in neural networks as a function of the number of parameters and number of training epochs.
To theoretically understand the trade-off between generalization error and an adversary's MI accuracy, we study the popular ``bi-level ensemble'' model that has been shown to exhibit benign overfitting in classification \citep{muthukumar2021classification, wang2021benign}. 
To characterize the difference of predictions on training points and test points, we leverage the proportional asymptotics regime, where precise analysis is enabled by tools such as the convex Gaussian min-max theorem \citep{thrampoulidis2018cgmt} and approximate message passing \citep{pmlr-v119-emami20a, gerbelot2020asymptotic}. 
In particular, we directly build upon \citet{salehi_logistic_2019} to analyze the behavior of logistic regression in the asymptotic regime.

\section{THEORETICAL FOUNDATIONS OF MEMBERSHIP INFERENCE}
\label{sec:theoretical}

We define our MI problem for classification as follows. Let $\mc{S} = ((\vx_i, y_i))_{i=1}^n$ be a training dataset of features $\vx_i \in \mc{X} \subseteq \reals^p$ and labels $y_i \in \mc{Y} = \{1, \ldots, k\}$ (i.e., multiclass classification). We assume that each data point and its associated label is an independent sample from a distribution $\mc{D}$ over the data such that $\mc{S} \sim \mc{D}^n$. Furthermore let $\mc{F}$ denote a class of machine learning models (e.g., linear models or neural networks) such that for $f \in \mc{F}$, $f: \mc{X} \rightarrow \reals^k$, producing a vector of confidence values from which the the final prediction is given as $\hat{y}(\vx) = \argmax_j [f(\vx)]_j$. For each pair $(\vx, y)$, we have access to a loss function $\ell: \mc{Y}\times \reals^k \rightarrow \reals_{\geq 0}$ that measures the performance of any $f \in \mc{F}$ on the data $\mc{S}$. 
The model's \textit{test (misclassification) error} is defined as $\error(f) = \Pr \paren{y \neq \hat{y}(\vx)}$, where $(\vx, y)$ is drawn from $\setD$ for our theoretical results or from the test set for our experiments.
Finally, let $A$ be a MI adversary.
For a fixed model $f \in \mc{F}$ trained on $\mc{S}$, we assume that $A: \mc{F}\times \setX \times \setY \rightarrow \{0, 1\}$ 
has access to $f$ and a sample $(\vx, y)$ and predicts $1$ if it believes $(\vx, y)\in \mc{S}$ and $0$ otherwise. 
To be rigorous,
we define MI as the following experiment \citep{yeom2018}.

\begin{experiment}\label{MI_experiment}
Given distribution $\mc{D}$, model class $\mc{F}$, loss function $\ell$, and adversary $A$, a membership inference experiment consists of the following:
\begin{enumerate}[topsep=-1pt,itemsep=-1pt]
    \item Sample $\mc{S} \sim \mc{D}^n$.
    \item Learn $\hat{f} \in \arg\min_{f\in \mc{F}} \sum_{i=1}^n \ell(y_i, f(\vx_i))$.
    \item Sample $m \in \{0, 1\}$ uniformly at random.
    \item If $m = 0$, sample a new test data point $(\vx, y) \sim \mc{D}$. If $m = 1$, sample a training data point $(\vx, y) \in \mc{S}$ uniformly at random.
    \item Observe the adversary's prediction $A(\hat{f}, \vx, y)$.
\end{enumerate}
\end{experiment}
In essence, Experiment~\ref{MI_experiment} reduces the problem of MI to one of hypothesis testing. Accordingly, we quantify the performance of an adversary in terms of its \emph{membership (inference) advantage}, defined as the difference between the adversary's true positive rate and the false positive rate.

\begin{definition}[\citealp{yeom2018}]
The membership advantage of an adversary $A$ against $\hat{f}$ is 
\begin{multline}
   \adv(A) = \Pr(A(\hat{f}, \vx, y) = 1 \mid m = 1) \\
   - \Pr(A(\hat{f}, \vx, y) = 1\mid m = 0), 
\end{multline}
where $\Pr(\cdot)$ is taken jointly over all randomness in Experiment~\ref{MI_experiment}.
\end{definition}

Membership inference can be performed successfully when the model treats points from the training dataset $\mc{S}$ ``differently'' than new test points. 
For instance, if the distribution 
of the model's output on a data point $(\vx, y)$ differs significantly when $(\vx, y)$ is a training point vs.\ when it is not, then MI attacks can distinguish between the two distributions to determine if $m=0$ or $m=1$.
Indeed, we observe in Figure \ref{fig:losses} that as a model trains on data, its loss on those data points decreases at a rate faster than its loss on non-training data points.
Then, even an attack as simple as thresholding the loss $\left(A(f, \vx, y) = \ind \set{\ell(y, f(\vx)) < \tau}\right)$ \citep{yeom2018, sablayrolles2019white} can successfully perform MI.

In this work, we consider single-query \emph{black-box adversaries}, which only have access to the data point $(\vx, y)$ and the model's output $\hat{f}(\vx)$ rather than the whole model. In this setting, the optimal attack is known to be the likelihood ratio test (LRT) \citep{sablayrolles2019white, carlini2021membership}: 

\begin{proposition}[\citealp{tan2022parameters}]
\label{proposition:optimal_adversary}
The adversary that maximizes membership advantage is:
\begin{align*}
    A^*(&\vx_0, \hat{f}(\vx_0)) \\&= \begin{cases} 
        1 &\mbox{if } P(\widehat{y_0} \mid m=1, \vx_0) > P(\widehat{y_0} \mid m=0, \vx_0),\\
        0 &\mbox{otherwise},
    \end{cases}
\end{align*}
where $\widehat{y_0} = \hat{f}(\vx_0)$ and $P$ denotes the distribution function for $\widehat{y_0}$ over the randomness in the membership inference experiment conditioned on $\vx_0$.
\end{proposition}

That is, given $(\vx, y)$ and $\hat{f}(\vx)$, the LRT adversary outputs 1 if the likelihood of the model outputting $\hat{f}(\vx)$ is higher if $(\vx, y)$ was a training point than if it was not a training point.

\subsection{Analysis Framework and Core Theoretical Result}

In this work, we theoretically analyze the role of parameters and regularization for MI against a regularized high-dimensional logistic regression model.
We define the logistic loss $\ell(y, z) = \rho(z) - yz$ in terms of the function $\rho(z) = \log(1 + \exp(z))$ whose derivative $\rho'(z) = 1 / (1 + \exp(-z))$ is the sigmoid function.
We let $\vx_i \sim \normal(\vzero, \frac{1}{p} \mSigma)$ for some positive definite covariance matrix $\mSigma \in \reals^{p \times p}$, and for ground truth coefficients $\vbeta^* \in \reals^p$, binary labels $y_i \in \set{0,1}$ are generated such that $\Pr(y_i = 1 | \vx_i) = \rho'(\vx_i^\transp \vbeta^*)$. Our learned decision function is $\hat{f}(\vx) = \vx^\transp \widehat{\vbeta}$, yielding predictions $\hat{y}(\vx) = \ind \{\hat{f}(\vx) > 0\}$, where
\begin{align}
    \label{eq:logistic-ridge}
    \widehat{\vbeta} = \argmin_\vbeta \frac{1}{n} \sum_{i=1}^n \ell(y_i, \vx_i^\transp \vbeta) + \frac{\lambda}{2p} \norm[2]{\vbeta}^2.
\end{align}
We study the accuracy of the LRT adversary in the asymptotic limit as $n, p \to \infty$ with $n/p \to \delta \in (0, \infty)$.
Being the optimal adversary, the LRT attack provides upper bounds on the membership advantage across single-query black-box adversaries (Proposition \ref{proposition:optimal_adversary}).
The asymptotic setting enables us to apply the analysis of \citet{salehi_logistic_2019}, who used the convex Gaussian min-max theorem (CGMT) \citep{thrampoulidis2018cgmt} to completely characterize the generalization performance of logistic regression in terms of the solution to a nonlinear system of equations of a few scalar variables; see Appendix~\ref{app:salehi-logistic} for details.

As observed in Proposition \ref{proposition:optimal_adversary}, analyzing the LRT requires a characterization of the distribution of model outputs for both training and test points. It is typically easy to characterize the test point output distribution because of the statistical independence between the model and the test point.  However, the distribution for the model's output on training points is much more difficult because the training procedure adds statistical dependence between the model and the training point that is nontrivial to address. Existing analyses from frameworks such as the CGMT are insufficient to give us the distributions of the outputs for a \emph{single} training point over the randomness of the remaining training dataset.

To address this, we provide a novel leave-one-out-based characterization of the distribution of the output of a linear model for any specific training point. We first recall the definition of the proximal operator, and we then provide the informal statement of our characterization with a more detailed version in Appendix~\ref{app:leave-one-out}.

\begin{definition}[Proximal operator]
The \emph{proximal operator} of a function $\Omega \colon \reals^p \to \reals$ is defined as
\begin{align}
    \prox{\Omega}{\vv} = \argmin_{\vw \in \reals^p} \Omega(\vw) + \frac{1}{2}\norm[2]{\vw - \vv}^2.
\end{align}
\end{definition}

\begin{theorem}[Informal version of Theorem~\ref{thm:leave-one-out-ridge}]
\label{thm:leave-one-out-ridge-informal}
Consider the solution $\widehat{\vbeta}$ to the optimization problem in \eqref{eq:logistic-ridge}. There exists $\gamma > 0$ such that in the limit as $n, p \to \infty$ with $n/p \to \delta \in (0, \infty)$, for any training point $\vx_i$,
\begin{align}
    \vx_i^\transp \widehat{\vbeta} \dconv \prox{\gamma \ell(y_i, \cdot)}{\vx_i^\transp \widehat{\vbeta}_{-i}},
\end{align}
where $\widehat{\vbeta}_{-i}$ is the solution to \eqref{eq:logistic-ridge} with $(\vx_i, y_i)$ omitted from the training set, and $\dconv$ denotes convergence in distribution where the randomness is over the other $n-1$ training points.
\end{theorem}

That is, the distribution of the model output for a training point is simply the distribution of the proximal operator of the loss function applied to the output of the training point as if it was a new test point. In essence, our theorem allows one to extend the ease of analyzing test points into the analysis of training points. Note also that the theorem shows how the model's loss for training points is driven closer to zero than for new test points, allowing an adversary to exploit this difference to perform MI as discussed above. We illustrate the strong match between the characterization in Theorem~\ref{thm:leave-one-out-ridge-informal} and the empirically obtained histograms for the output of a logistic regression model for practically sized problems in Figure~\ref{fig:output-histograms}.

\begin{figure*}[ht]
    \centering
    \includegraphics[width=\textwidth]{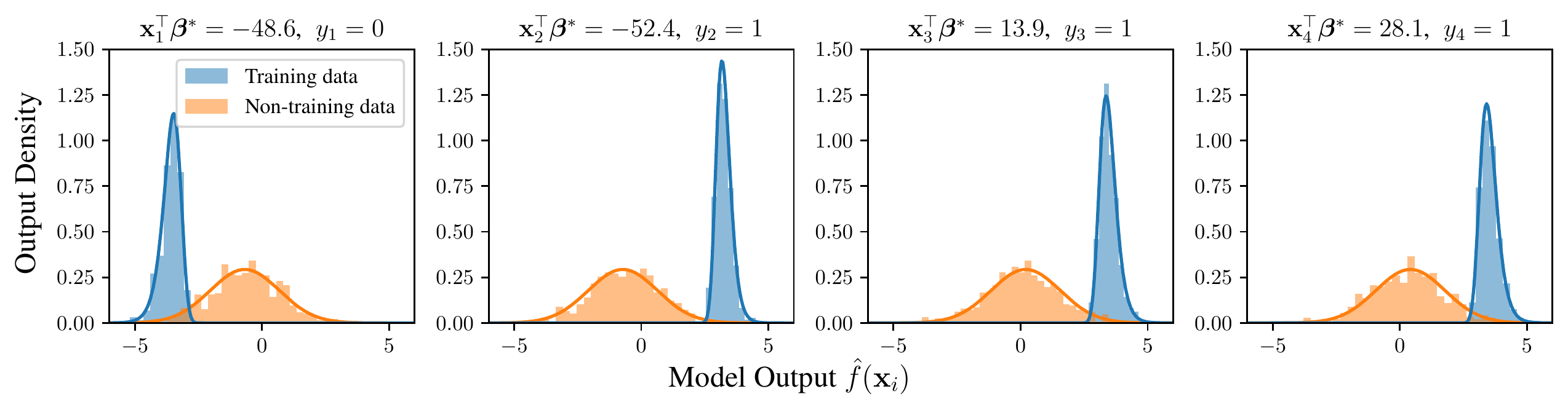}
    \vspace{-0.75cm}
    \caption{\textbf{Theoretical distributions match empirical observations.} We plot theoretical densities (solid line) according to Theorem~\ref{thm:leave-one-out-ridge-informal} (see details in Appendix~\ref{app:membership-advantage}) versus empirical histograms of a logistic regression model's output on a given sample $\vx_i$ when it is a test or a training point for four different fixed points ($i \in \set{1, 2, 3, 4}$) and fixed $\vbeta^*$ drawn from a bi-level ensemble with $n/d = 1/2$, $\phi = 3$, $\sigma_\beta = 50$, $\lambda = 0.1$, $\eta = 1$. Empirical histograms are outputs over 500 trials where $n=500$ additional random training points are used to train a logistic regression model on bi-level ensemble features either including or not including $(\vx_i, y_i)$. As we can see, the training point outputs are pulled toward the training labels.}
    \label{fig:output-histograms}
\end{figure*}

We strongly believe this theoretical tool to be of independent interest, opening the door to future theoretical study of privacy in high dimensional linear models, in particular with sharp asymptotics for any given adversary rather than simply worst-case bounds. Our proof strategy is general and applies to general convex losses and regularization penalties, as we describe in Appendix~\ref{app:leave-one-out}. A particularly exciting open question for future work is determining what types of losses, regularization, and feature distributions can lead to a small $\gamma$ such that the resulting model is the most private.

\subsection{A Bi-level Feature Ensemble} 

In order to study the trade-off between accuracy and privacy as a function of overparameterization in machine learning models, we need a setting in which benign overfitting occurs---that is, that as we increase the number of parameters of our model, generalization accuracy increases as well. To that end, we define a bi-level feature ensemble similar to that considered by \citet{muthukumar2021classification, wang2021benign}. 
In this model, we define $\mSigma$ and $\vbeta^*$ for some $d < p$ and $\eta > 0$ as
\begin{align}
    \label{eq:bilevel-model}
    \begin{aligned}
    [\mSigma]_{k, k'}^2 &= \begin{cases}
        \frac{p}{d} &\text{if } 1 \leq k = k' \leq d, \\
        \frac{\eta p}{p-d} & \text{if } d < k = k' \leq p, \\
        0 & \text{ if } k \neq k',
    \end{cases}
    \\
    \beta_k^* &\sim
    \begin{cases}
        \normal(0, \sigma_\beta^2) &\text{if } 1 \leq k \leq d, \\
        0 & \text{if } d < k \leq p.
    \end{cases}
    \end{aligned}
\end{align}
In this way, there is always a total variance of $1$ in the first $d$ features and of $\eta$ in the tail of $p-d$ features. As $\phi = p / d \to \infty$, this model is known to exhibit benign overfitting \citep{wang2021benign}.

The intuition behind this feature model is that the signal $\vbeta^*$ is fundamentally low dimensional and is aligned with a small subset of $d$ highly representative features. Meanwhile, there are an abundance of nuisance features of very small magnitude that are uncorrelated with the signal, such that they can absorb label noise \citep{bartlett2020benign} without adversely affecting prediction on new examples with uncorrelated nuisance features. In this way, training points can achieve perfect accuracy even under noise while the model still generalizes well. Furthermore, nonlinearities like those used in neural networks are known to add a similar low-magnitude tail of nonzero eigenvalues to the feature covariance in their Gaussian equivalents~\citep{pennington2017nonlinear,mei2022generalization}, connecting this feature model with realistic models like neural networks.

\subsection{Asymptotic Privacy and Utility}

Given the framework of the CGMT, we can easily determine the asymptotic generalization error for logistic regression~\citep{salehi_logistic_2019}. Thanks to Theorem~\ref{thm:leave-one-out-ridge-informal}, we can also determine the MI advantage given an adversary $A$. 
The following corollary captures these results, specializing the MI advantage to that of the worst-case optimal LRT adversary.

\begin{corollary}
\label{cor:error-and-mi}
Consider the bi-level feature ensemble in \eqref{eq:bilevel-model} and the decision function $\hat{f}(\vx) = \vx^\transp \widehat{\vbeta}$ for $\widehat{\vbeta}$ solving \eqref{eq:logistic-ridge}.
Then there exist $\alpha, \gamma, \sigma > 0$ such that, in the limit as $p \to \infty$ with $n/p \to \delta \in (0, \infty)$,
\begin{enumerate}[label=(\roman*)]
    \item \emph{Generalization error.} The misclassification error for a new test pair $(\vx, y)$ is given by 
    \begin{align}
        \error(\hat{f}) = \expect{\rho'(Z) \Phi(-\tfrac{\alpha Z}{\sigma})},
    \end{align}
    where $\Phi$ is the standard normal CDF and $Z \sim \normal(0, \sigma_\beta^2)$;
    \item \emph{Membership advantage.}
    For any training pair $(\vx_i, y_i)$, the membership advantage of the optimal adversary is given by
    \begin{multline}
        \max_{A} \adv(A,\hat{f}; \vx_i, y_i) \\
        = \tfrac{1}{\sigma} \int_\reals \max \big\{
        \Phi' \big( 
        \tfrac{z - \alpha \vx_i^\transp \vbeta^* + \gamma(\rho'(z) - y_i) }{\sigma} 
        \big) 
        (1 + \gamma \rho''(z)) \\
        - \Phi' \big( 
        \tfrac{z - \alpha \vx_i^\transp \vbeta^*}{\sigma} 
        \big) 
        , 0\big\} dz,
    \end{multline}
    where $\Phi'$ is the standard normal PDF.
\end{enumerate}
\end{corollary}

\begin{figure*}[t]
	\centering
	\includegraphics[width=\textwidth]{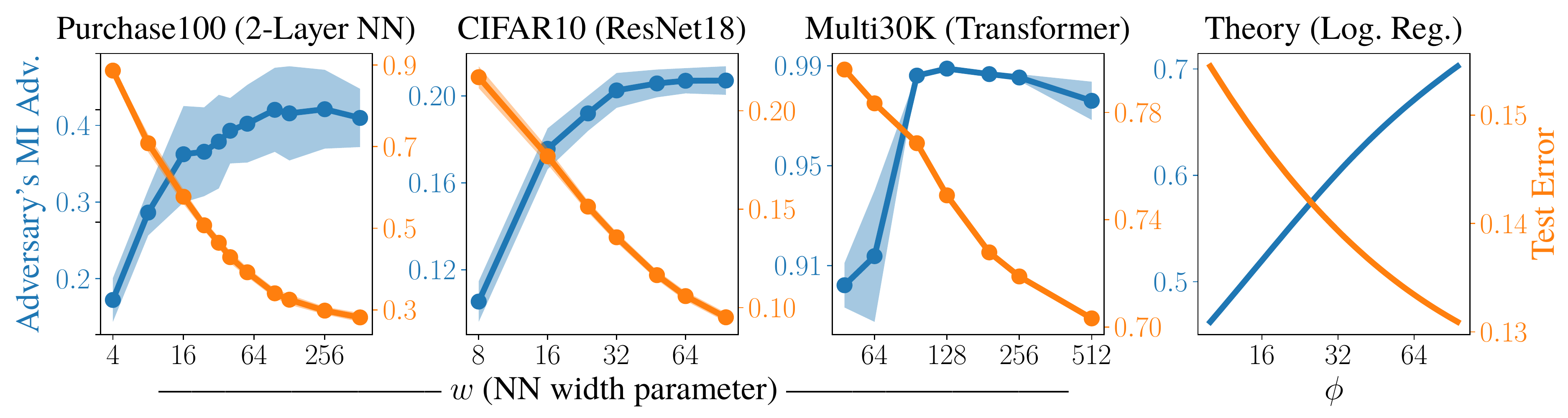}
	\caption{\textbf{Privacy vs.\ parameters.} For NNs trained to optimal early stopping with respect to validation error, we show cases where increasing the network's width generally increases MI advantage on the network even as test error decreases. We see a similar effect for logistic regression with the bi-level ensemble theoretically when $\lambda$ is tuned to minimize test error.}
	\label{fig:MI_vs_params_nn}
\end{figure*}

It is not possible to determine closed-form expressions for $(\alpha, \gamma, \sigma)$ in terms of the parameters $(\lambda, \delta, \phi, \eta, \sigma_\beta)$ of the regularized bi-level feature ensemble estimator in general, as the former are the solutions to a system of nonlinear equations (see Theorem~\ref{thm:salehi-logistic} in Appendix~\ref{app:salehi-logistic}). 
This makes direct theoretical analysis of the privacy--utility trade-offs difficult.

However, we can obtain the values $(\alpha, \gamma, \sigma)$ by solving the nonlinear system numerically.\footnote{In addition to describing this procedure in Appendix~\ref{app:salehi-logistic}, we also provide our code at \url{https://github.com/tanjasper/benign_overparam_MI}.} In the next sections, when we plot theoretical trade-off curves for logistic regression, we solve the nonlinear system and then evaluate the above expressions using numerical integration, reporting the average sample-specific membership advantage $\E_{(\vx_i, y_i) \sim \setD}[\max_A \adv(A, \hat{f}; \vx_i, y_i)]$.  We refer the reader to Appendix~\ref{app:membership-advantage} for proof details for Corollary~\ref{cor:error-and-mi}, where we also derive the expressions for the density functions for the bi-level feature ensemble that we plot in Figure~\ref{fig:output-histograms}.
\section{INDIVIDUAL PRIVACY--UTILITY TRADE-OFFS}

We now present multiple scenarios that demonstrate privacy--utility trade-offs as a function of either the number of model parameters or the amount of regularization, individually. Specifically, when either increasing the number of parameters or decreasing the amount of regularization from an over-regularized state, the resulting machine learning model becomes more accurate (improved generalization performance) but becomes less private (higher adversary MI advantage). 
The increase in accuracy with overparameterization has been discussed in detail in the double descent literature \citep{belkin2019reconciling, nakkiran2021deep, dar2021farewell}.
The decrease of MI privacy with overparameterization has been observed for linear regression models by \cite{tan2022parameters}, but we show that the phenomenon is robust, extending to classification models and even highly nonlinear models such as deep NNs.
We show parameter-wise and regularization-wise tradeoffs experimentally on various machine learning tasks and provide some theoretical insights to their origins.
Experimental details not in the main text can be found in Appendix~\ref{sec:experimental_setup}.
Shaded areas in NN plots indicate one standard deviation over repeated trials.

\subsection{Parameter-Wise Privacy--Utility Trade-Off}
\label{sec:params_tradeoff_nn}

In Figure~\ref{fig:MI_vs_params_nn}, we consider a variety of neural networks and plot both the adversary's membership advantage and the NN's test error as a function of the NN's width (number of parameters). We observe how MI increases (thus damaging privacy) while test error decreases (yielding a more accurate model) as the number of parameters grows. Here, we consider NNs that are trained with optimal (with respect to validation error) early stopping: we stop training at the number of training epochs that maximizes validation accuracy. We consider three machine learning tasks: feature vector classification on the Purchase100 dataset \citep{shokri2017membership} using a 2-layer NN, image classification on CIFAR10 \citep{krizhevsky2009learning} using the ResNet18 architecture \citep{he2016deep}, and language translation on the Multi30K dataset \citep{elliott2016multi30k} using the Transformer architecture \citep{vaswani2017attention}. We control the number of parameters of the networks by scaling the size of the hidden dimensions by a width parameter $w$. 
The MI attack we employ is the sample-specific loss threshold attack (``attack R'' of \citealp{ye2021enhanced}): $A(f, \vx, y) = \ind \set{\ell(y, f(\vx)) < \tau(\vx, y)}$, where $\tau(\vx, y)$ is a sample-specific threshold learned for each data point over reference/shadow models.
We also include similar experiments demonstrating the same phenomenon for support vector machines (SVMs) in Appendix \ref{sec:svm}.

\begin{figure*}[t]
	\centering
	\includegraphics[width=\textwidth]{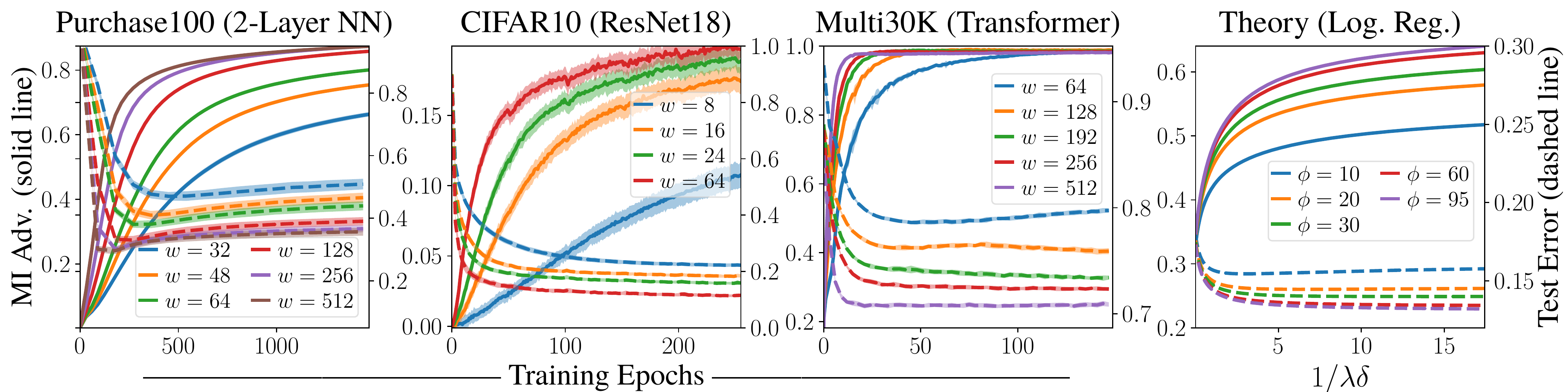}
	\caption{\textbf{Privacy vs.\ regularization.} Regardless of neural network width (parameterized by $w$), increasing the number of training epochs (decreasing regularization) increases the adversary's MI advantage (solid line) while simultaneously decreasing its test error (dashed line). This induces a regularization-wise privacy--utility trade-off. The same holds theoretically for logistic regression when decreasing ridge regularization under the bi-level feature ensemble setting.
 }
	\label{fig:epochwise}
\end{figure*}

\paragraph{Theoretical Insights.}
\label{sec:params-theory}

Using our theoretical tool from Theorem~\ref{thm:leave-one-out-ridge-informal}, we can in fact prove that for an extremely broad class of settings, including the bi-level ensemble which exhibits benign overfitting, extreme overparameterization leads to perfect MI by any loss-thresholding adversary.
We capture this result in the following theorem.
We have omitted some technical conditions related to the convergence of a system of fixed point equations for the statement of part (a); please see Theorem~\ref{thm:logistic-high-dim-mi-formal} in Appendix~\ref{app:proof-logistic-high-dim-mi} for precise details.  
In the theorem statements, we assume all scalar variables (such as $\lambda$ and $\eta$) to be fixed unless otherwise specified.
\begin{theorem}
\label{thm:logistic-high-dim-mi}
    If $\hat{f}(\vx) = \vx^\transp \widehat{\vbeta}$, where $\widehat{\vbeta}$ is the solution to \eqref{eq:logistic-ridge}, and for some $\tau > 0$ we have an adversary $A(f, \vx, y) = \ind \set{\ell(y, f(\vx)) < \tau}$, then as $n, p \to \infty$ with $n/p \to \delta \in (0, \infty)$,
    \begin{enumerate}[label=(\alph*)]
        \item If $\lim_{p \to \infty} \norm[2]{\mSigma^{1/2} \vbeta^*} /\sqrt{p}$ exists and is finite, and $\liminf_{p \to \infty} \lambda_{\min}(\mSigma) > 0$, where $\lambda_{\min}(\mSigma)$ is the smallest eigenvalue of $\mSigma$, then as $\delta \to 0$, $\adv(A) \to 1$.
        \item For the bi-level ensemble in \eqref{eq:bilevel-model}, if $p / d \to \phi \in (1, \infty)$ and $d/n$ converges to a fixed value, then as $\phi \to \infty$, $\adv(A) \to 1$, and in the limit as $\lambda \to \infty$, $\error(\hat{f})$ is decreasing in $\phi$.
    \end{enumerate}
\end{theorem}
This theorem highlights that as $\delta \to 0$ (the model becomes increasingly overparameterized), \emph{any} constant-threshold adversary's MI advantage converges to 1, yielding perfect MI attacks on the learned model. We emphasize that the constant-threshold adversary is much weaker than the sample-specific loss threshold adversary we consider in our experiments, and it need not be adapted to the problem in any way, yet overparameterized models are still vulnerable. This is true regardless of any (fixed) value of regularization strength, meaning that ridge regularization is not sufficient to protect against MI attacks, echoing the observation of \citet{tan2022parameters} in linear regression. This result applies not only to standard isotropic data covariances, but also to highly anisotropic covariances such as the bi-level ensemble.

Part (b) highlights how in the right circumstances, we can still see generalization performance improving with overparameterization---that there is a trade-off between generalization and privacy, just as in our experimental results.
We illustrate this alongside neural networks in Figure~\ref{fig:MI_vs_params_nn}  for the bi-level model with fixed $n / d = 5$, $\sigma_\beta = 10$, and $\eta = 1$, with $\lambda$ tuned to minimize test error, analogously to the optimal validation error early stopping in the NN experiments. This plot is generated using the expressions in Corollary~\ref{cor:error-and-mi} for test error and the optimal adversary's MI advantage. We see that the generalization error decreases but the adversary's MI advantage increases as the length of the tail of small eigenvalues of $\mSigma$ increases for larger values of $\phi$.

\subsection{Regularization-Wise Privacy--Utility Trade-Off}
\label{sec:epochwise}

Using the same classification tasks and NN architectures as in Section \ref{sec:params_tradeoff_nn}, we empirically demonstrate an epoch-wise privacy--utility trade-off in Figure \ref{fig:epochwise}, where we plot the adversary's MI advantage and the model's generalization error as a function of training epochs. Stopping training at earlier epochs corresponds to higher regularization, as the model has less opportunity to overfit to training data. We include a variety of NN widths in our plot, demonstrating similar trade-offs across widths.

We also plot the theoretical test error and MI advantage from Corollary~\ref{cor:error-and-mi} for logistic regression with the bi-level feature ensemble as a function of the regularization strength. Specifically, we plot the regularization as a function of $1 / \lambda \delta$, where $\lambda$ is the $\ell_2$ regularization parameter, such that smaller values of $1 / \lambda \delta$ correspond to more regularization. Just as we explore a variety of widths for NNs, we consider a variety of values of $\phi = p/d$, measuring the amount of overparameterization for the bi-level feature ensemble.

Interestingly, Figure \ref{fig:epochwise} shows how the adversary's MI advantage can continue to increase with epochs even if test error stays the same. Thus, generalization error does not completely characterize MI. 
Instead, it is the increasing generalization (cross-entropy) loss gap that leads to increased MI advantage.
As the NN is trained for more epochs, or the logistic regression model is less regularized, training loss decreases at a greater rate than test loss, making it easier to divide the training and test losses with a loss threshold, as illustrated in Figure~\ref{fig:losses}. The losses continue separating even after test error has converged, causing the MI advantage to continue to increase.

\section{A BLESSING OF DIMENSIONALITY: ELIMINATING THE PRIVACY--UTILITY TRADE-OFF
}\label{sec:joint_tune}

\begin{figure*}[t]
	\centering
	\includegraphics[width=\textwidth]{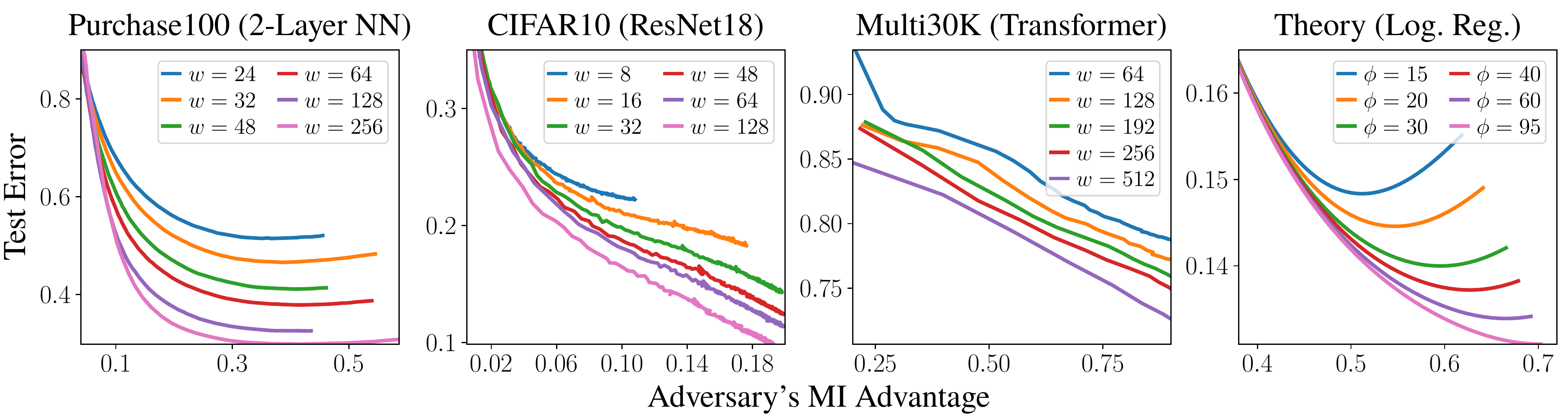}
	\caption{\textbf{Trade-offs are better with increased width.}
	We plot the regularization-wise privacy--utility trade-off of networks of different widths by sweeping through different numbers of training epochs. Observe that, for sufficiently low validation errors, wider networks are closer to the lower-left (high accuracy, high privacy) region compared to narrower networks. The same holds theoretically for logistic regression in the bi-level feature ensemble sweeping through the ridge penalty.}
	\label{fig:tradeoff}
\end{figure*}

We now show that, perhaps counter-intuitively, if we jointly tune both the numbers of parameters and the amount of regularization, we can eliminate the privacy--utility trade-off.
The main idea is to \emph{increase} the number of parameters while also \emph{increasing} the regularization appropriately.

\begin{figure*}[t]
	\begin{subfigure}{0.49\textwidth}
    	\includegraphics[width=\textwidth]{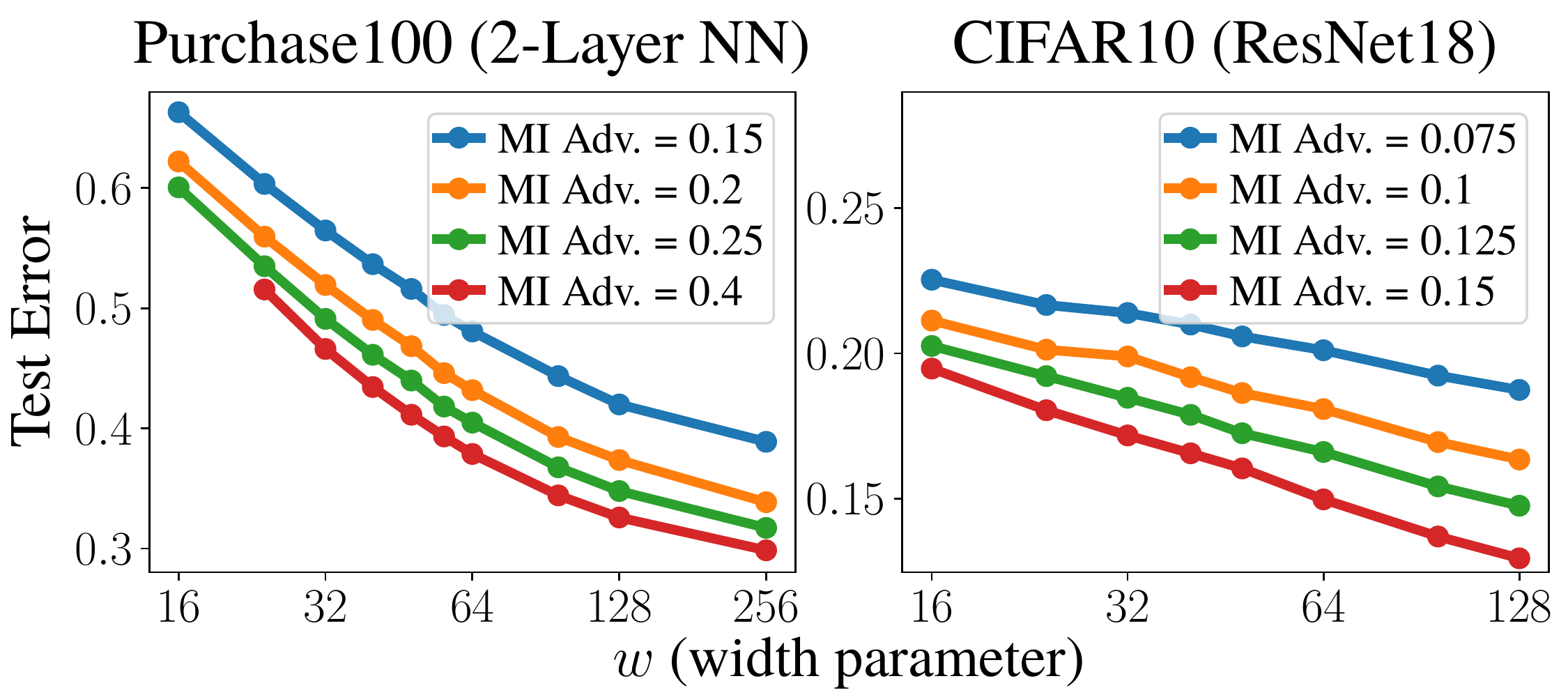}
    	\caption{Test error vs.\ network width for fixed MI adv.}
    	\label{fig:fixed_MI}
	\end{subfigure}
	\hspace*{\fill}
	\begin{subfigure}{0.49\textwidth}
    	\includegraphics[width=\textwidth]{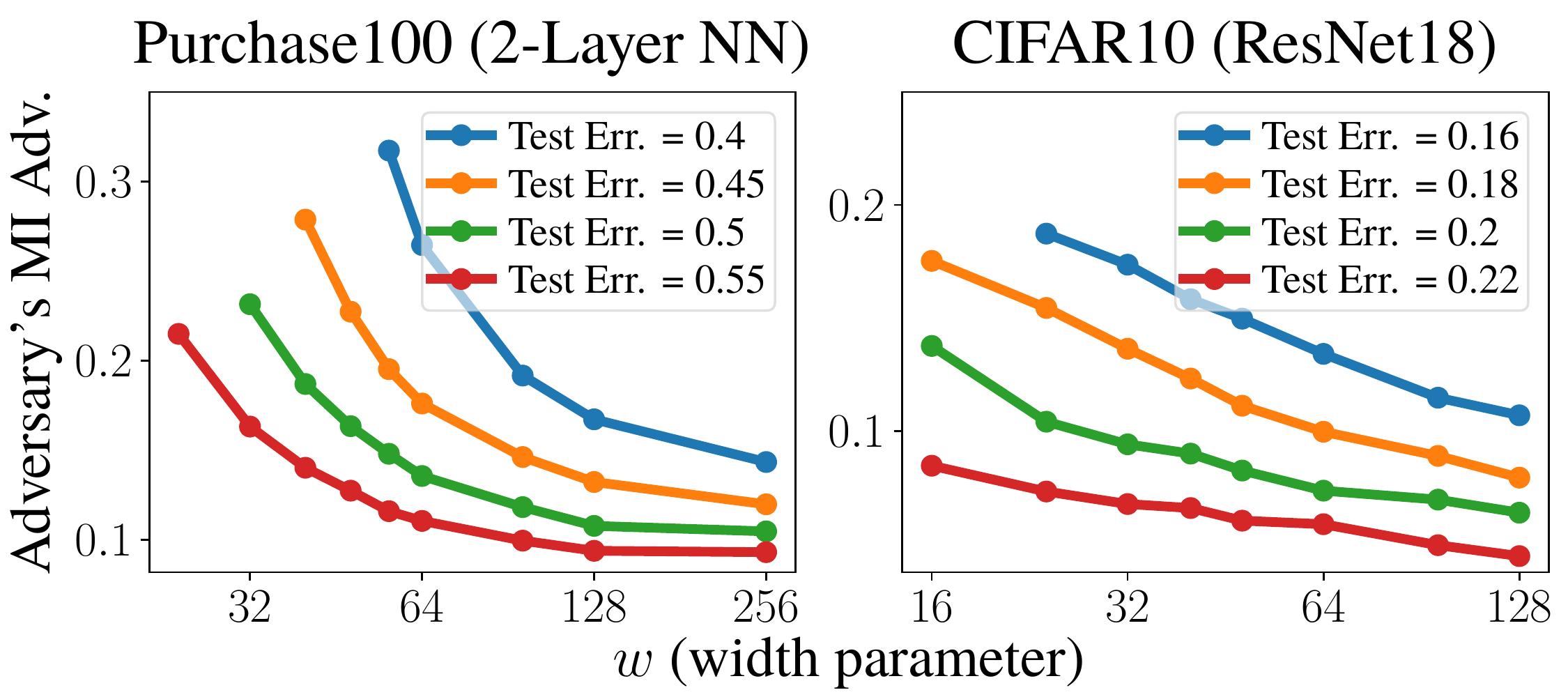}
    	\caption{MI adv.\ vs.\ network width for fixed test error.}
    	\label{fig:fixed_val}
	\end{subfigure}
	\vspace{8pt}
	\caption{\textbf{Overparameterization with early stopping eliminates the privacy--utility trade-off.} (a)~For each network width, we train the network until it reaches a given MI advantage value. We then plot the test error of the networks. Observe how test error decreases with parameters at a fixed MI advantage value, showing how proper tuning of parameters and epochs together improves model accuracy without damaging its privacy. Thus, this eliminates the privacy--utility trade-off. (b) Same as (a) but switching the roles of MI advantage and test error.}
	\label{fig:fixed_MI_val}
\end{figure*}

Our key observation is that the 
decrease in the model's generalization error and the increase in an adversary's MI advantage occur at different rates during training for NNs of different widths (recall Figure~\ref{fig:epochwise}). 
However, it is difficult to compare these rates across different NN widths when privacy and utility are individually plotted against regularization.
Hence, we plot parametric curves for varying widths as a function of regularization (epochs for NNs, and ridge penalty for logistic regression) in a {\em privacy--utility plane} in Figure~\ref{fig:tradeoff}, which enables us to abstract away the regularization strength and compare trade-off curves across widths directly.
In the plot, ideal performance is the lower-left corner, as this represents low MI advantage (high privacy) and low test error. 
In this representation, the story becomes clear: wider networks can induce better privacy--utility trade-offs. That is, they are both below and to the left of the trade-off curves for narrower networks. The same occurs for theoretical logistic regression with the bi-level ensemble. Thus, increased parameterization is not inherently a privacy liability and can instead actually improve the privacy of a model.

We explicitly show how early stopping (with the appropriate stopping rule) \emph{eliminates} the privacy--utility trade-off for overparameterization in Figure~\ref{fig:fixed_MI_val}. If we tune the number of training epochs for each width such that a fixed MI advantage is reached (which takes fewer epochs for larger widths), then we see from Figure~\ref{fig:fixed_MI} that overparameterization only \emph{decreases} the generalization error. Similarly, tuning the number of epochs to a fixed validation error results in a decrease of the adversary's MI advantage with increasing width, as shown in Figure~\ref{fig:fixed_val}.
In essence, either privacy or improved generalization can be obtained without taking a hit in the other by opting for a larger network with proper regularization.
While we do not recommend early stopping alone as a sufficient privacy-preserving mechanism (practitioners should likely also consider the wide collection of existing MI defense schemes), this strongly suggests that practitioners should include wider networks in their model search and then tune their regularization appropriately to achieve a desired level of privacy.

In Appendix~\ref{sec:additional-experiments}, we include additional experiments that we could not include in the main paper for space reasons, including a version of Figure~\ref{fig:fixed_MI_val} for Transformers on Multi30K, repeating all of the experiments in Figures~\ref{fig:MI_vs_params_nn}--\ref{fig:fixed_MI_val} for global loss thresholding attacks, and the MI vs.\ test error trade-off for networks trained with DP-SGD~\citep{abadi2016deep} on CIFAR10. In all cases, we see the same behavior---when the regularization is tuned for MI, larger models achieve better protection from MI and better classification accuracy than smaller models.
\section{DISCUSSION}\label{sec:conc}

We began this exploration with a question: is overparameterization a privacy liability? In our theoretical and empirical investigation, we have demonstrated cases wherein overparameterization \emph{can} be a privacy risk, but that it \emph{need not be}, and that, in fact, it can provide even greater privacy when coupled with appropriate regularization. To the best of our knowledge we have provided the first study of this effect in the context of membership inference. 
While our work shows a number of common scenarios where larger models coupled with regularization achieve greater privacy, we acknowledge that we do not prove the generality of this phenomenon.
We encourage further investigation into this topic to better understand how universal this blessing of dimensionality is.

While we showed how ridge regularization for logistic regression and early stopping for neural networks bring out this blessing of dimensionality, many other types of regularization are used in practice. 
For one example, we include a preliminary experiment using DP-SGD \citep{abadi2016deep} in Figure~\ref{fig:dpsgd} in Apppendix~\ref{sec:additional-experiments}, for which we also observe wider networks having better trade-offs.
However, not every regularizer may induce the same effect, and an interesting open research direction is to discover which types of regularization or other learning techniques can draw out even more privacy benefits from large models. 
For example, in the field of differential privacy, by fine-tuning pre-trained language models, \citet{yu2021differentially,li2022largelanguage} achieve better accuracy with larger models than smaller ones for the same privacy budget. 
A nascent regularization approach strongly worth further study is network pruning, which has been observed to be an effective defense against membership inference attacks~\citep{wang2020against} as well as a vulnerability in some settings~\citep{yuan2022membership}.

The phenomenon of better privacy--utility trade-offs for overparameterized models also has important takeaways for our general understanding of the benefits of overparameterization.
As we have shown, highly overparameterized models not only have more capacity to memorize than smaller networks (which leads to increased risk of MI), but they also appear to learn the underlying structure of the data \emph{even more quickly} than they memorize data.
Identifying the mechanism that provides this benefit in overparameterized models and developing appropriate measures for an ``effective'' number of parameters that reflects the memorization capacity of the model as a function of both the true number of parameters and forms of regularization are important open questions.
We believe our leave-one-out characterization of the training output distribution in Theorem~\ref{thm:leave-one-out-ridge-informal} 
will be helpful in answering these questions with respect to privacy.

\smallskip
\subsection*{Acknowledgements}
This work was supported by NSF grants CCF-1911094, IIS-1838177, and IIS-1730574; ONR grants N00014-18-12571, N00014-20-1-2534, and MURI N00014-20-1-2787; AFOSR grant FA9550-22-1-0060; and a Vannevar Bush Faculty Fellowship, ONR grant N00014-18-1-2047.


\bibliographystyle{unsrtnat}
\bibliography{refs}

\clearpage
\onecolumn

\appendix

\section{Limitations and Considerations}\label{sec:limit_broad}

\subsection{Limitations of this work}
A possible limitation of this work is that we focus on a particular class of inference attacks, the loss threshold attack, in most of our experimental results. More subtly, the procedure we propose for estimating membership inference vulnerability involves computing an empirical estimate. As such, there is uncertainty in this process. In practical settings where the training and validation sets are large, this is likely not a major concern. That said, in settings where the privacy budget is very low and/or privacy is paramount it may additionally be necessary to use high-probability bounds on the adversary's MI advantage rather than the estimate directly for an added layer of security. Furthermore, the theoretical guarantees are in the asymptotic regime. While they show a strong correlation with finite dimension experiments (e.g., Figure~\ref{fig:epochwise}), developing tight, non-asymptotic results is an open question. \cite{wang2021benign}, for instance, are able to derive non-asymptotic guarantees to connect generalization error to overparameterization, but the same technique does not apply in the case of membership inference: it is important to consider the distribution of the model's output for specific inputs---not just the population on average. 

\subsection{Ethical Considerations}

Ensuring that models protect the data that they are trained on is important for modern machine learning systems. In order to achieve benign overparameterization for membership inference and generalization error jointly, we perform precise tuning and early stopping. When implementing these ideas in practical scenarios, it is recommended that a sensitivity analysis additionally be conducted to ensure that the chosen parameters are sufficiently tight. Without doing so, applying this method may lead to false confidence in a method's robustness to MI attacks. 
In general, the authors believe that in settings where privacy is of the utmost concern, such as when training with medical data, additional measures beyond those covered in this work should be taken to ensure that the data stays private. Finally, this paper focuses on membership inference in particular and these results are not as general as complete differential privacy. Practitioners should consider additional privacy vulnerabilities beyond membership inference alone. 
\section{Background material}

Here we include a few definitions and results borrowed from other works.

\subsection{Definitions}
\label{app:definitions}

We again define the proximal operator for a function $\Omega$ as follows.
\begin{definition}[Proximal operator]
The \emph{proximal operator} of a function $\Omega \colon \reals^p \to \reals$ is defined as
\begin{align}
    \prox{\Omega}{\vv} = \argmin_{\vw \in \reals^p} \Omega(\vw) + \frac{1}{2}\norm[2]{\vw - \vv}^2.
\end{align}
\end{definition}
It will be valuable to consider the first-order optimality condition of the proximal operator; for differentiable penalties, the minimizer $\vw^*$ satisfies
\begin{align}
    \nabla \Omega(\vw^*) + \vw^* - \vv = \vzero.
\end{align}
For our work, we will need the form of the scalar proximal operator for $\Omega(\vv) = \frac{1}{2} \norm[2]{\mA \vv}^2$ for symmetric $\mA \in \reals^{p \times p}$, which for $t > 0$ is given by
\begin{align}
    \prox{t\Omega}{v} = \inv{\mI_p + t \mA^2} \vv.
\end{align}

We also have the definition of local Lipschitzness from \citet{salehi_logistic_2019}.
\begin{definition}[Locally Lipschitz]
A function $\Phi \colon \reals^d \to \reals$ is said to be \emph{locally Lipschitz} if $\forall M > 0$, $\exists L_M \geq 0$, such that $\forall \vx, \vy \in [-M, +M]^d$, $|\Phi(\vx) - \Phi(\vy)| \leq L_M \norm{\vx - \vy}$.
\end{definition}

\subsection{Fixed point equations for logistic regression}
\label{app:salehi-logistic}

We borrow the following theorem (slightly adapted to our notation) from \citet{salehi_logistic_2019}.
\begin{theorem}[Theorem 1 of \citealp{salehi_logistic_2019}]
\label{thm:salehi-logistic}
For training data  $\vx_i \overset{\iid}{\sim} \normal(\vzero, \frac{1}{p}\mI_p)$ and $y_i \sim \mathrm{Bernoulli}(\vx_j^\transp \vbeta^*)$, consider the optimization program
\begin{align}
    \widehat{\vbeta} = \argmin_{\vbeta \in \reals^p} \frac{1}{n} \sum_{i=1}^n \ell(y_i, \vx_i^\transp \vbeta) + \frac{\lambda}{p} \Omega(\vbeta),
\end{align}
where $\ell(y, z) = \rho(z) - yz$ for $\rho(z) =  \log(1 + \exp(-z))$ is the logistic loss, and $\Omega \colon \reals^p \to \reals$ is a convex regularization function. Consider also the following nonlinear system
\begin{equation}
    \label{eq:salehi-fixed-point}
    \left\{
    \begin{aligned}
    \kappa^2 \alpha &= \frac{1}{p} {\vbeta^*}^\transp \prox{\lambda \sigma \tau \Omega}{\sigma \tau (\theta \vbeta^* + \frac{r}{\sqrt{\delta}} \vg)}, \\
    \gamma &= \frac{1}{r \sqrt{\delta} p} \vg^\transp \prox{\lambda \sigma \tau \Omega}{\sigma \tau (\theta \vbeta^* + \frac{r}{\sqrt{\delta}} \vg)}, \\
    \kappa^2 \alpha^2 + \sigma^2 &= \frac{1}{p} \norm[2]{ \prox{\lambda \sigma \tau \Omega}{\sigma \tau (\theta \vbeta^* + \frac{r}{\sqrt{\delta}} \vg)}}^2, \\
    \gamma^2 &= \frac{2}{r^2} \expect{\rho'(-\kappa Z_1) \paren{\kappa \alpha Z_1 + \sigma Z_2 - \prox{\gamma \rho}{\kappa \alpha Z_1 + \sigma Z_2} }^2}, \\
    \theta \gamma &= -2 \expect{\rho''(-\kappa Z_1) \prox{\gamma \rho}{\kappa \alpha Z_1 + \sigma Z_2}}, \\
    1 - \frac{\gamma}{\sigma \tau} &= \expect{\frac{2 \rho'(-\kappa Z_1)}{1 + \gamma \rho''\paren{\prox{\gamma \rho}{\kappa \alpha Z_1 + \sigma Z_2}}}},
    \end{aligned}
    \right.
\end{equation}
where $\vg \sim \normal(\vzero, \mI_p)$ is independent of $\vbeta^*$ and $\Omega$, and $Z_1$ and $Z_2$ are independent standard normal variables. Assume that as $p \to \infty$, $n / p \to \delta$, $\smallnorm[2]{\vbeta}/\sqrt{p} \to \kappa$, and that the system in \eqref{eq:salehi-fixed-point} has a unique solution $(\bar{\alpha}, \bar{\sigma}, \bar{\gamma}, \bar{\theta}, \bar{\tau}, \bar{r})$. Then, as $p \to \infty$, for any locally-Lipschitz function $\Psi \colon \reals \times \reals \to \reals$, we have
\begin{align}
    \frac{1}{p} \sum_{j=1}^p \Psi(\hat{\beta}_j, \beta_j^*) \pconv \frac{1}{p} \sum_{j=1}^p \Psi([\mGamma(\vbeta^*, \vg)]_j, \beta_j^*),
\end{align}
where $\mGamma(\vv, \vz) = \prox{\lambda \bar{\sigma} \bar{\tau} \Omega}{\bar{\sigma} \bar{\tau} (\bar{\theta} \vv + \frac{\bar{r}}{\sqrt{\delta}} \vz)}$.
\end{theorem}
The astute reader may note that \citet{salehi_logistic_2019} require separable regularizers and drawing $\vbeta^*$ element-wise $\iid$ from some distribution, but that neither of these are required for their proof technique to go through, so we have stated the more general result here, as we will need both of these assumptions to be relaxed. 

For a given problem, we can obtain the limiting solution $(\bar{\alpha}, \bar{\sigma}, \bar{\gamma}, \bar{\theta}, \bar{\tau}, \bar{r})$ by iterating the system of fixed point equations \eqref{eq:salehi-fixed-point}. That is, we can compute all six right hand sides via numerical integration, then obtain the corresponding values of $(\alpha, \sigma, \gamma, \theta, \tau, r)$ according to the expressions on the left-hand side, and then plugging these values back into the right-hand side and repeating until convergence.

\section{Leave-one-out analysis for membership inference}
\label{app:leave-one-out}

In order to study MI attacks, we need to understand how the distribution of training points differs from test points. We prove the following result to this end for logistic regression with a ridge penalty; however, the proof strategy is general and applies readily to other losses and penalties for general linear models that admit a result similar to Theorem~\ref{thm:salehi-logistic}, which includes many common models in machine learning \citep{thrampoulidis2018cgmt,pmlr-v119-emami20a,gerbelot2020asymptotic}.
\begin{theorem}
\label{thm:leave-one-out-ridge}
Consider the solution $\widehat{\vbeta}$ to the optimization problem in \eqref{eq:logistic-ridge}. Let $\widetilde{\vbeta}^* = \mSigma^{1/2} \vbeta$, $\widetilde{\vx}_i = \mSigma^{-1/2} \vx_i$, and $\widetilde{\Omega}(\widetilde{\vbeta}) = \frac{1}{2}\smallnorm[2]{\mSigma^{-1/2} \widetilde{\vbeta}}^2$. Assume Theorem~\ref{thm:salehi-logistic} holds for $\widetilde{\vbeta}^*$ in place of $\vbeta^*$ and $\widetilde{\Omega}$ in place of $\Omega$. Then for any training point $\vx_i$,
\begin{align}
    \vx_i^\transp \widehat{\vbeta} \dconv \prox{\bar{\gamma} \ell(y_i, \cdot)}{\vx_i^\transp \widehat{\vbeta}_{-i}},
\end{align}
where $\bar{\gamma}$ is from the result of Theorem~\ref{thm:salehi-logistic}, and
\begin{align}
    \widehat{\vbeta}_{-i} = \argmin_{\vbeta \in \reals^p} \frac{1}{n} \sum_{i' \neq i}^n \ell(y_{i'}, \vx_{i'}^\transp \vbeta) + \frac{\lambda}{2p} \norm[2]{\vbeta}^2.
\end{align}
\end{theorem}
\begin{proof}
We first make a leave-one-out modification the optimization problem for a general loss and regularizer:
\begin{align}
    \widehat{\vbeta} = \mSigma^{-1/2} \cdot \argmin_{\widetilde{\vbeta}} \frac{1}{n} \sum_{i' \neq i}^n \ell(y_{i'}, \widetilde{\vx}_{i'}^\transp \widetilde{\vbeta}) + \frac{\lambda}{p} \overline{\Omega}(\widetilde{\vbeta}),
\end{align}
where 
\begin{align}
    \overline{\Omega}_i(\widetilde{\vbeta}) = \widetilde{\Omega}(\widetilde{\vbeta}) + \frac{1}{\lambda \delta} \ell(y_i, \widetilde{\vx}_i^\transp \widetilde{\vbeta}).
\end{align}
Applying Theorem~\ref{thm:salehi-logistic} to this problem, the solution is  equivalent to one of the form
\begin{align}
    \widehat{\vbeta}_\mathrm{equiv} = \mSigma^{-1/2} \cdot \prox{t \overline{\Omega}_i}{a \widetilde{\vbeta}^* + b \vg},
\end{align}
where $t = \lambda \bar{\sigma} \bar{\tau}$, $a = \bar{\sigma} \bar{\tau} \bar{\theta}$, and $b = \bar{\sigma} \bar{\tau} \bar{r} / \sqrt{\delta}$. This proximal operator is the solution $\vw^*$ to the equation
\begin{align}
    t \nabla \widetilde{\Omega}(\vw^*) + \frac{t}{\lambda \delta} \ell'(y_i, \widetilde{\vx}_i^\transp \vw^*) \widetilde{\vx}_i + \vw^* - (a \widetilde{\vbeta}^* + b \vg) = 0,
\end{align}
where $\ell'(y_i, z) = \partial \ell(y_i, z) / \partial z$. Note that this is equivalent to
\begin{align}
    \vw^* = \prox{t \widetilde{\Omega}}{a \widetilde{\vbeta}^* + b \vg - \frac{t}{\lambda \delta} \ell'(y_i, \widetilde{\vx}_i^\transp \vw^*) \widetilde{\vx}_i}.
\end{align}
Here we specialize to the ridge penalty, but this can be extended to separable regularizers with careful application of Stein's lemma. Plugging in the form of the proximal operator for generalized ridge penalties, we have
\begin{align}
    \vw^* = \mSigma \inv{\mSigma + t \mI_p} \paren{a \widetilde{\vbeta}^* + b \vg - \frac{t}{\lambda \delta} \ell'(y_i, \widetilde{\vx}_i^\transp \vw^*) \widetilde{\vx}_i}.
\end{align}
We wish to characterize $\vx_i^\transp \widehat{\vbeta}$, which is equivalent to characterizing $\vx_i^\transp \widehat{\vbeta}_\mathrm{equiv} = \widetilde{\vx}_i^\transp \vw^*$. Firstly, we note that for any random vector $\vu$ such that $\smallnorm[2]{\vu}^2 / \sqrt{p} \to C_\vu < \infty$ that is independent of $\widetilde{\vx}_i$,
\begin{align}
    \frac{1}{p} \vu^\transp \vw^* \pconv \frac{1}{p} \vu^\transp \mSigma \inv{\mSigma + t \mI_p} \paren{a \widetilde{\vbeta}^* + b \vg}.
\end{align}
Appealing to Theorem~\ref{thm:salehi-logistic} again, this means that the nonlinear system is in fact unaffected by our leave-one-out modification asymptotically, and that both cases have the same solution $(\bar{\alpha}, \bar{\sigma}, \bar{\gamma}, \bar{\theta}, \bar{\tau}, \bar{r})$ to the nonlinear system \eqref{eq:salehi-fixed-point}. Therefore,  
\begin{align}
    \vx_i^\transp \widehat{\vbeta}_{-i} \dconv \widetilde{\vx}_i^\transp \mSigma \inv{\mSigma + t \mI_p} \paren{a \widetilde{\vbeta}^* + b \vg} \sim \normal(0, \kappa^2 \bar{\alpha}^2 + \bar{\sigma}^2).
\end{align}
Since $\vg/\sqrt{p}$ and $\widetilde{\vx}_i$ have the same distribution, from the second equation in the nonlinear system \eqref{eq:salehi-fixed-point} we know that
\begin{align}
    \widetilde{\vx}_i^\transp \mSigma \inv{\mSigma + t \mI_p} \widetilde{\vx}_i \asconv \frac{1}{p} \vg^\transp \mSigma \inv{\mSigma + t \mI_p} \vg = \frac{\bar{\gamma} \delta}{\bar{\sigma} \bar{\tau}}.
\end{align}
All together, this gives us
\begin{align}
    \vx_i^\transp \widehat{\vbeta} &\pconv \widetilde{\vx}_i^\transp \mSigma \inv{\mSigma + t \mI_p} \paren{a \widetilde{\vbeta}^* + b \vg} - \bar{\gamma} \ell'(y_i, \vx_i^\transp \widehat{\vbeta}) \\
    \implies 
    \vx_i^\transp \widehat{\vbeta} &\dconv \prox{\bar{\gamma} \ell(y_i, \vx_i^\transp \widehat{\vbeta})}{\vx_i^\transp \widehat{\vbeta}_{-i}},
\end{align}
which is the stated result.
\end{proof}

\section{Formal version of Theorem~\ref{thm:logistic-high-dim-mi} and proof}
\label{app:proof-logistic-high-dim-mi}

Theorem~\ref{thm:logistic-high-dim-mi} is a slightly informal version of the following theorem. The only difference is technical, as we must assume the convergence of the nonlinear system~\eqref{eq:salehi-fixed-point} for part (a). The convergence of MI advantage to 1 of part (b) of Theorem~\ref{thm:logistic-high-dim-mi} is implied by part (a) of the following theorem.
\begin{theorem}
\label{thm:logistic-high-dim-mi-formal}
Consider the solution $\widehat{\vbeta}$ to the optimization problem in \eqref{eq:logistic-ridge}. Then
\begin{enumerate}[label=(\alph*)]
    \item If the result of Theorem~\ref{thm:leave-one-out-ridge} holds and the minimum eigenvalue of $\mSigma$ is lower bounded by a positive constant for sufficienly small $\delta$, then as $\delta \to 0$, $\adv(A) \to 1$.
    \item For the bilevel model in \eqref{eq:bilevel-model}, if $p / d \to \phi \in (1, \infty)$ and $d/n$ converges to a fixed value, then in the limit as $\lambda \to \infty$, $\error(f)$ is decreasing in $\phi$.
\end{enumerate}
\end{theorem}

This theorem makes claims of two natures: that MI advantage of the adversary goes to 1, and that generalization error is decreasing. For the former, we will show that the output distributions diverge for train and test points such that it becomes trivial to distinguish between the two distributions, and for the latter, we will determine the form of the generalization error and show that it is decreasing in the proposed variable.

\subsection{Part (a): membership inference advantage}

We will assume the notation and setting from the proof of Theorem~\ref{thm:leave-one-out-ridge}. When rewriting equations from \eqref{eq:salehi-fixed-point}, we will omit the bars (e.g., $\bar{\gamma}$ in the next section) when describing general implications of the equations, and then use bars to describe conclusions about the \emph{unique} fixed point solution that characterizes the limiting estimator, which we assumed to exist in applying Theorem~\ref{thm:salehi-logistic}.

\subsubsection{Growth of $\bar{\gamma}$}

First, we show that $\bar{\gamma}$, the scaling factor of the proximal operator in Theorem~\ref{thm:leave-one-out-ridge}, tends to infinity as $\delta \to 0^+$. This will drive training points to be much different from test points as long as the test point distribution variance doesn't increase.
From the sixth equation in the nonlinear system~\eqref{eq:salehi-fixed-point}, since the right hand side is greater than 0 and the fixed point variables are non-negative, we can conclude that $\sigma \tau > \gamma$. We can combine this with the second equation to yield 
\begin{align}
    \gamma &= \frac{1}{p} \vg^\transp \mSigma \inv{\mSigma + \lambda \sigma \tau \mI_p} \vg \frac{\sigma \tau}{\delta} \\
    &= \frac{1}{\lambda \delta p} \vg^\transp \inv{\frac{1}{\lambda \sigma \tau} \mI_p + \mSigma^{-1}} \vg \\
    &> \frac{1}{\lambda \delta p} \vg^\transp \inv{\frac{1}{\lambda \gamma} \mI_p + \mSigma^{-1}} \vg \\
    &\asconv \frac{1}{\lambda \delta p}
    \tr \bracket{\inv{\frac{1}{\lambda \gamma} \mI_p + \mSigma^{-1}}}
\end{align}
Because the smallest eigenvalue $\lambda_{\min{}}(\mSigma) > 0$, this implies that
\begin{align}
    \lambda \gamma > \frac{1}{\delta} \frac{1}{\frac{1}{\lambda \gamma} + \frac{1}{\lambda_{\min{}}(\mSigma)}} \implies \frac{\lambda \gamma}{\lambda_{\min{}}(\mSigma)} > \frac{1}{\delta} - 1.
\end{align}
Therefore, asymptotically, there exists a constant $c_{\bar{\gamma}} > 0$ such that for sufficiently small $\delta$, we have $\lambda \bar{\gamma} \geq c_{\bar{\gamma}} / \delta$, so $\bar{\gamma} \to \infty$ as $\delta \to 0^+$.

\subsubsection{Vanishing of output variance.}

We next argue that $\kappa^2 \bar{\alpha}^2 + \bar{\sigma}^2$ tends to 0 as $\delta \to 0$. We remind the reader that as in the proof of Theorem~\ref{thm:leave-one-out-ridge}, this is the variance of $\vx_i^\transp \widehat{\vbeta}_{-i}$, which is also equal to the variance of the output for an unseen test point.

First, we consider the fourth equation in the nonlinear system~\eqref{eq:salehi-fixed-point}. Applying the first-order optimality condition of the proximal operator, this is equivalent to
\begin{align}
    \label{eq:r-squared-upper-bound}
    r^2 = 2 \expect{\rho'(-\kappa Z_1) \rho' \paren{\prox{\gamma \rho}{\kappa \alpha Z_1 + \sigma Z_2}}^2} \leq 2.
\end{align}
Similarly, the fifth equation can be written as
\begin{align}
    \theta &= \frac{-2}{\gamma} \expect{\rho''(-\kappa Z_1) \paren{\kappa \alpha Z_1 + \sigma Z_2 - \gamma \rho' \paren{\prox{\gamma \rho}{\kappa \alpha Z_1 + \sigma Z_2}}}} \\
    \label{eq:theta-upper-bound}
    &= 2 \expect{\rho''(-\kappa Z_1) \rho' \paren{\prox{\gamma \rho}{\kappa \alpha Z_1 + \sigma Z_2}}} \\
    & \leq \frac{1}{2},
\end{align}
where we have used the fact that the expectation of any odd function of a standard normal variable is zero, and that $\rho''(u) \leq 1/4$ for all $u \in \reals$. Thus, both $r$ and $\theta$ are upper bounded by constants. Let us now consider the third equation.
\begin{align}
    \kappa^2 \alpha^2 + \sigma^2
    &= \frac{(\sigma \tau)^2}{p} 
    (\theta \widetilde{\vbeta}^* + \frac{r}{\sqrt{\delta}} \vg)^\transp
    \mSigma^2 \paren{\mSigma + \lambda \sigma \tau \mI_p}^{-2}
    (\theta \widetilde{\vbeta}^* + \frac{r}{\sqrt{\delta}} \vg) \\
    \label{eq:variance-upper-bound}
    &\leq \frac{1}{\lambda^2} \paren{\kappa^2 \theta^2 + \frac{r^2}{\delta}} \\
    &\leq \frac{1}{\lambda^2} \paren{4 \kappa^2 + \frac{1}{4\delta}}.
\end{align}
Here the first inequality is obtained by letting $\sigma \tau$ tend to infinity, and the second is obtained by applying our upper bounds for $\theta$ and $r$.
Therefore, for sufficiently small $\delta$, there exists $c_1$ such that
$\kappa^2 \alpha^2 + \sigma^2 \leq c_1^2 / \delta$.

We now wish to return to \eqref{eq:r-squared-upper-bound} and \eqref{eq:theta-upper-bound} to determine tighter upper bounds. To that end, we first prove the following lemma
\begin{lemma}
\label{lem:prox-logistic-bound}
Let $Z$ be a standard normal random variable. For any $a_0, b_0 > 0$, there exist $\delta_0 > 0$ and $c > 0$ such that for all $a \geq a_0$, $b \leq b_0$, and $0 < \delta < \delta_0$,
\begin{align}
    \Pr \paren{\prox{a \rho / \delta}{\frac{b Z}{\sqrt{\delta}}} > \log \paren{c \delta \log (1  / \delta)}} \leq \delta^2.
\end{align}
\end{lemma}
\begin{proof}
We begin by observing that is sufficient to prove the claim for $a = a_0$ and $b = b_0$, since the probability is monotonically decreasing and increasing, respectively, in each variable for sufficiently small $\delta$.
By standard Gaussian tail bounds, for sufficiently small $\delta$, 
\begin{align}
    \Pr(Z > 4 \log (1/\delta)) \leq \delta^2.
\end{align}
The proximal operator is a strictly increasing function of $Z$, so we can determine the bound on its tail by determining an upper bound on $\prox{a \rho / \delta}{\frac{4b \log (1/\delta))}{\sqrt{\delta}}}$.
The first-order optimality condition for the proximal operator is
\begin{align}
    w^* = \frac{4b \log (1/\delta))}{\sqrt{\delta}} - \frac{a}{\delta} \rho'(w^*).
\end{align}
It is clear that for sufficiently small $\delta$, $w^* < 0$, since $\rho'(u) \geq 1/2$ for $u \geq 0$. Therefore, since $\rho'(u) = e^u / (1 + e^u)$, there exists $c_\delta \in (1/2, 1)$ such that $\rho'(w^*) = c_\delta e^{w^*}$. We can then solve for and bound $w^*$ for some $c > 0$ and sufficiently small $\delta$ as
\begin{align}
    w^* &= \frac{4b \log (1/\delta))}{\sqrt{\delta}} - W_0\paren{ \frac{a c_\delta}{\delta} \exp\paren{\frac{4b \log (1/\delta))}{\sqrt{\delta}}}} \\
    &\leq -\log \paren{\frac{a c_\delta}{\delta}} + \log \paren{\frac{4b \log (1/\delta))}{\sqrt{\delta}} + \log \paren{\frac{a c_\delta}{\delta}}} \\
    &\leq \log ( c \delta \log(1/\delta)),
\end{align}
where $W_0$ is the principal branch of the Lambert $W$ function, and the first inequality follows from the lower bound $W_0(x) \geq \log x - \log \log x$ for $x \geq e$.
Let $\delta_0$ be a sufficiently small so that the above arguments hold, and the claim is proved.
\end{proof}

Applying Lemma~\ref{lem:prox-logistic-bound} with $a_0 = c_{\bar{\gamma}}$ and $b_0=c_1$ to \eqref{eq:r-squared-upper-bound}, we can use the facts that $\rho'(u) \leq 1$ and that $\rho'(u) \leq e^{u}$ to obtain for some $c_r > 0$
\begin{align}
    r^2 \leq 2 \paren{c_r^2 \delta^2 \log^2(1 / \delta) + \delta^2}.
\end{align}
Thus for some $c_{\bar{r}} > 0$, $\bar{r} \leq c_{\bar{r}} \delta \log(1/\delta)$ for sufficiently small $\delta$. We then apply Lemma~\ref{lem:prox-logistic-bound} to \eqref{eq:theta-upper-bound} to similarly obtain for some $c_\theta > 0$
\begin{align}
    \theta \leq \frac{1}{2} \paren{c_\theta \delta \log (1 / \delta) + \delta^2}
\end{align}
Thus for some $c_{\bar{\theta}} > 0$, $\bar{\theta} \leq c_{\bar{\theta}} \delta \log(1/\delta)$ for sufficiently small $\delta$. Therefore, returning again to \eqref{eq:variance-upper-bound}, there exists some $c_2 > 0$ such that for sufficiently small $\delta$,
\begin{align}
    \kappa^2 \bar{\alpha}^2 + \bar{\sigma}^2 \leq c_2^2 \delta \log^2 (1 / \delta).
\end{align}
Hence the output variance tends to zero as $\delta \to 0^+$.

\subsubsection{Membership inference advantage}

We wrap up the proof by proposing two more lemmas for the proximal operator of the logistic loss
\begin{lemma}
\label{lem:prox-logistic-limit}
Fix $C > 0$. For all $v$ such that $|v| < C$ and $y \in \set{0, 1}$,
\begin{align}
    \lim_{a \to \infty} |\prox{a \ell(y, \cdot)}{v}| = \infty \text{ uniformly},
\end{align}
where $\ell(y, z) = \log(1 + \exp(z)) - yz$ is the logistic loss.
\end{lemma}
\begin{proof}
The proximal operator $\prox{a \ell(y, \cdot)}{v}$ is the unique solution $w \in \reals$ to the equation
\begin{align}
    w = v + a(y - \rho'(w)).
\end{align}
Consider $y = 1$, and suppose the claim was not true. Then there exists $c_1 > 0$ such that for all $a_0 > 0$, there exists $a > a_0$ and $v \in (-C, C)$ such that $|w| < c_1$. Let $c_2 = \rho'(c_1)$. This implies that
\begin{align}
    c_1 + a c_2 > v + a.
\end{align}
Since $c_2 < 1$, this inequality does not hold for any $a > a_0$ if $a_0$ is sufficiently large, leading to a contradiction. The case for $y = 0$ is entirely analogous if we make the substitution $\rho'(w) = 1 - \rho'(-w)$.
\end{proof}

\begin{lemma}
\label{lem:prox-logistic-threshold}
Let $Z$ be a standard normal random variable. Then for any $\tau > 0$, if $a_n$ and $b_n$ are sequences such that as $n \to \infty$, $a_n \to \infty$ and $b_n \to 0$, then
\begin{align}
    \lim_{n \to \infty} \Pr\paren{|\prox{a_n \ell(y, \cdot)}{b_n Z}| > \tau} - \Pr\paren{|b_n Z| > \tau} = 1,
\end{align}
\end{lemma}
\begin{proof}
For sufficiently large $n$, by a standard tail bound for Gaussian variables, with probability at least $1 - e^{-(\tau / b_n)^2 / 2}$, we know that $|b_n Z| < \tau$. Again for sufficiently large $n$, we know that $|\prox{a_n \ell(y, \cdot)}{b_n Z}| > \tau$ for all $|b_n Z| < \tau$ by Lemma~\ref{lem:prox-logistic-limit}.
Thus,
\begin{align}
    \Pr\paren{|\prox{a_n \ell(y, \cdot)}{b_n Z}| > \tau} - \Pr\paren{|b_n Z| > \tau} \geq 1 - 2 e^{-(\tau / b_n)^2 / 2},
\end{align}
which tends to $1$ as $n \to \infty$.
\end{proof}

Applying Lemma~\ref{lem:prox-logistic-threshold} to our problem, using the fact that $\bar{\gamma} \to \infty$ and $\kappa^2 \bar{\alpha}^2 + \bar{\sigma}^2 \to 0$, we see that any adversary that applies a threshold $|\hat{f}(\vx)| > \tau$ for a fixed threshold $\tau$ will achieve MI advantage of 1 as $\delta \to 0$. Any loss-based fixed-threshold adversary inherits this behavior, as for the logistic loss, $\ell(y, \hat{f}(\vx))$ is a monotonically decreasing function of $|\hat{f}(\vx)|$, so thresholding the loss is equivalent to thresholding the magnitude of the model output.

\subsection{Part (b): test accuracy for the bi-level ensemble}

In the bi-level ensemble, when applying Theorem~\ref{thm:salehi-logistic} for $\widetilde{\vbeta}^*$ in place of $\vbeta^*$, asympotically, the first three equations in the nonlinear system~\eqref{eq:salehi-fixed-point} become
\begin{equation}
\left\{
    \begin{aligned}
    \kappa^2 \alpha &= \frac{\sigma \tau \theta \kappa^2}{1 + \frac{\lambda \sigma \tau}{\phi}}, \\
    \gamma &= \frac{\sigma \tau}{\delta} \paren{\frac{1}{\phi + \lambda \sigma \tau} + \frac{\phi - 1}{\phi + \frac{\lambda \sigma \tau}{\eta} (\phi - 1)}}, \\
    \kappa^2 \alpha^2 + \sigma^2 &= \frac{(\sigma \tau \theta \phi \kappa)^2 + (\sigma \tau r)^2 \frac{\phi}{\delta}}{(\phi + \lambda \sigma \tau)^2} + \frac{(\sigma \tau r)^2 \frac{\phi}{\delta} (\phi - 1)}{(\phi + \frac{\lambda \sigma \tau}{\eta} (\phi - 1))^2}.
    \end{aligned}
    \right.    
\end{equation}
As we discussed in the proof of part (a), $r$ and $\theta$ are always upper bounded by constants, so as $\lambda \to \infty$, regardless of the behavior of $\sigma \tau$, the left-hand sides of all three equations tend to zero. For this reason, applying our reformulations of the proximal operators and taking appropriate limits, the last three equations in the nonlinear system become
\begin{equation}
\left\{
    \begin{aligned}
    r^2 &= \frac{1}{4}, \\
    \theta &= \expect{\rho''(-\kappa Z_1)}, \\
    \sigma \tau &= 4.
    \end{aligned}
    \right.    
\end{equation}
These simplifications largely result from applying $\rho'(0) = 1/2$ and appealing to symmetry arguments. The final equation results from the algebraic manipulation
\begin{align}
    \frac{\gamma}{\sigma \tau} &= \expect{2 \rho'(-\kappa Z_1) \paren{1 - \frac{1}{1 + \gamma \rho''\paren{\prox{\gamma \rho}{\kappa \alpha Z_1 + \sigma Z_2}}}}} \\
    &= \expect{2 \rho'(-\kappa Z_1) \frac{\gamma \rho''\paren{\prox{\gamma \rho}{\kappa \alpha Z_1 + \sigma Z_2}}}{1 + \gamma \rho''\paren{\prox{\gamma \rho}{\kappa \alpha Z_1 + \sigma Z_2}}}}.
\end{align}
Now knowing that $\bar{\sigma} \bar{\tau} = 4$, we can consider very large $\lambda \to \infty$ to obtain
\begin{equation}
\left\{
    \begin{aligned}
    \alpha &= \frac{\theta \phi}{\lambda} + o \paren{\tfrac{1}{\lambda}}, \\
    \gamma &= \frac{2}{\lambda \delta} + o \paren{\tfrac{1}{\lambda}}, \\
    \kappa^2 \alpha^2 + \sigma^2 &= \frac{1}{\lambda^2} \paren{(\theta \phi \kappa)^2 + \frac{\phi}{4 \delta} \paren{1 + \frac{\eta^2 }{\phi - 1}}} + o \paren{\tfrac{1}{\lambda}}.
    \end{aligned}
    \right.
\end{equation}
Generalization error equals $\Pr\paren{y \oplus \ind \set{\vx^\transp \widehat{\vbeta} > 0} = 1}$, where $\oplus$ is the exclusive or operator, which by symmetry we can compute as
\begin{align}
    \Pr\paren{y \oplus \ind \set{\vx^\transp \widehat{\vbeta} > 0} = 1} &= 2 \Pr \paren{y = 0, \bar{\alpha} \vx^\transp \vbeta^* + \bar{\sigma} Z > 0} \\
    &= 2 \expect[\vx]{\Pr(y = 0 | \vx^\transp \vbeta^*) \Phi \paren{\frac{\vx^\transp \vbeta^*}{\bar{\sigma} / \bar{\alpha}}}},
    \label{eq:misclassification-error-expectation}
\end{align}
where $\Phi \colon \reals \to [0, 1]$ is the standard normal CDF, and $Z$ is a standard normal random variable. It can be shown that this is decreasing in $\alpha / \sigma$, and from the above, in the limit as $\lambda \to \infty$,
\begin{align}
    \frac{\bar{\alpha}^2}{\bar{\sigma}^2} = \frac{4 \theta^2 \frac{\delta}{\phi}}{1 + \frac{\eta^2}{\phi - 1}},
\end{align}
which is increasing in $\phi$ for fixed $d/n = \delta / \phi$.

\section{Proof of Corollary~\ref{cor:error-and-mi}}
\label{app:membership-advantage}

\begin{proof}
The generalization error result immediately follows from \eqref{eq:misclassification-error-expectation} the previous section, since 
\begin{align}
\kappa^2 = \big\| \widetilde{\vbeta}^* \big\|_2^2 / p = \vbeta^{*\transp} \mSigma \vbeta^* / p \to \sigma_\beta^2. 
\end{align}

For membership advantage, we know from the previous section, Theorem~\ref{thm:leave-one-out-ridge}, and Theorem~\ref{thm:salehi-logistic} that the predictions on training and test points follow
\begin{align}
    \vx_i^\transp \widehat{\vbeta} \dconv \prox{\bar{\gamma}\ell(y_i, \cdot)}{\bar{\alpha} Z_i + \bar{\sigma} W},
    \quad
    \vx^\transp \widehat{\vbeta} \dconv \bar{\alpha} Z + \bar{\sigma} W,
\end{align}
where $\vx_i^\transp \vbeta^* \dconv Z_i$, $\vx^\transp \vbeta^* \dconv Z$, and $W \sim \normal(0, 1)$ is independent of $Z_i$ or $Z$. Here randomness is over the training dataset, so for a fixed $\vbeta^*$ and $\vx_i$ (or $\vx$), we have a fixed $Z_i$ (or $Z$).
Suppose the adversary is given some $\vx'$ and its (noisy) training label $y'$. If $\vx'$ (with corresponding $Z'$) is \emph{not} a training point,
\begin{align}
    \mu_\test (\hat{z} | \vx = \vx', y = y')
    & = \mu_\test(\hat{z} | \vx = \vx') \\
    &= \mu_W\paren{\frac{\hat{z} - \bar{\alpha}Z'}{\bar{\sigma}}} \frac{1}{\bar{\sigma}}.
\end{align}
The first equality is from the independence of the model output and the unused training label, and the second equality comes by the change of variables formula for scalar random variables in terms of $\mu_W$, which is a standard normal Gaussian density. 

If $\vx'$ is a training point, we have the following probability density:
\begin{align}
    \mu_\train (\hat{z} | \vx = \vx', y = y') =
    \mu_W\paren{\frac{g_y(\hat{z}) - \bar{\alpha}Z'}{\bar{\sigma}}} \frac{g_y'(\hat{z})}{\bar{\sigma}}.
\end{align}

Here $g_y(\cdot)$ is the inverse of $\prox{\bar{\gamma}\ell(y, \cdot)}{\cdot}$, which by the first-order optimality condition is
\begin{align}
    g_y(z) = z + \bar{\gamma} (\rho'(z) - y),
    \quad
    g_y'(z) = 1 + \bar{\gamma} \rho''(z).
\end{align}
We remind the reader that $\rho''(z) = \rho'(z) (1 - \rho'(z))$. Therefore, the densities can be easily evaluated by numerical integration.

Since the adversary is given the value of the loss, which is monotonic in $\hat{f}(\vx')$, and knows predicted label $\hat{y}(\vx')$, the adversary is equivalent to an adversary based on $\hat{f}(\vx')$ with the densities described above. The optimal adversary is given by
\begin{align}
    A^*(f, \vx', y') = \ind \set{\mu_\train(\hat{f}(\vx') | \vx = \vx', y = y') > \mu_\test(\hat{f}(\vx') | \vx = \vx', y = y')},
\end{align}
and we can compute its MI advantage specific to $(\vx', y')$ as 
\begin{align}
    \adv(A^*, \hat{f}; \vx', y')
    = \int_\reals \max \set{\mu_\train(z | \vx = \vx', y = y') - \mu_\test(z | \vx = \vx', y = y'), 0} dz,
\end{align}
Additionally, we can numerically evaluate this integral, and then we can compute the average sample-specific membership inference advantage as
\begin{align}
    \adv(A^*, \hat{f}) = \expect[\vx', y']{\adv(A^*, \hat{f}; \vx', y')},
\end{align}
which we can easily compute by numerical integration over the Gaussian density of $Z'$ and the fact that $\Pr(y' = 1 | \vx') = \rho'(Z')$.
\end{proof}

\section{Neural network experimental setup}
\label{sec:experimental_setup}

This section provides details on the NN experiments whose results are shown in Figures \ref{fig:MI_vs_params_nn}, \ref{fig:epochwise}, \ref{fig:tradeoff}, and \ref{fig:fixed_MI_val}. Unless otherwise specified, we use the default hyperparameters and initalizations of Pytorch implementations. The NN experiments are run on our internal servers with the following GPUs: NVIDIA TITAN X (Pascal), NVIDIA GeForce RTX 2080 Ti, NVIDIA TITAN RTX, and NVIDIA A100. The choice of which particular GPU is used for each experiment is decided only based on availability of the GPUs in our internal servers.

\subsection{The MI attack}

The MI attack employed in these experiments is the loss-threshold attack \citep{yeom2018, sablayrolles2019white, ye2021enhanced}. Given a trained NN $f$, the data point of interest $\mathbf{z}_0 = (\mathbf{x}_0, y_0)$, and a loss function $\ell$, the prediction $A(f(\mathbf{x}_0), \mathbf{z}_0)$ of this attack is given by:
\begin{align}
    A(f(x_0), y) = \begin{cases} 
        1 &\mbox{if } \ell(y_0, f(\mathbf{x}_0)) < \tau_{\mathbf{z}_0}\\
        0 &\mbox{otherwise}
    \end{cases},
\end{align}
where $\tau_{\mathbf{z}_0}$ is a calibrated threshold. The threshold is learned with the following procedure. Given a full training dataset $\mc{D}$, we train $n_\text{shadow}$ shadow models on random subsamples of this dataset such that for each $\mathbf{z}_0$ in the full dataset, some models are trained on datasets including $\mathbf{z}_0$ and the rest are trained on datasets that do not include $\mathbf{z}_0$. The shadow models have the same architecture and training procedure as the target models that will be attacked. Let $n_{\text{shadow}, \mathbf{z}_0, m=1}$ and $n_{\text{shadow}, \mathbf{z}_0, m=0}$ denote the (random) numbers of shadow models trained on $\mathbf{z}_0$ and not trained on $\mathbf{z}_0$, respectively.
We then evaluate all these shadow models on $\mathbf{z}_0$ and collect all loss values of the shadow models trained on $\mathbf{z}_0$ into a vector $\mathbf{s}_{\mathbf{z}_0, m=1}$ and the loss values of the shadow models not trained on $\mathbf{z}_0$ into a vector $\mathbf{s}_{\mathbf{z}_0, m=0}$. The membership advantage $\text{Adv}_\text{shadow}$ of a threshold $\hat{\tau}_{\mathbf{z}_0}$ is given by:
\begin{align}
    \text{Adv}_{\text{shadow}, \mathbf{z}_0} = \tfrac{\left|\set{s \in \mathbf{s}_{\mathbf{z}_0, m=1} \colon s < \tau_{\mathbf{z}_0}}\right|}{n_{\text{shadow}, \mathbf{z}_0, m=1}}
        - \tfrac{\left|\set{s \in \mathbf{s}_{\mathbf{z}_0, m=0} \colon s < \tau_{\mathbf{z}_0}}\right|}{n_{\text{shadow}, \mathbf{z}_0, m=0}},
\end{align}
which is simply the difference of the empirical true positive rate and false positive rate. Note that there are many optimal thresholds that maximize $\text{Adv}_\text{shadow}$. Indeed, if $\hat{\tau}_{\mathbf{z}_0}$ is one such optimal threshold, then so is any $\tau \in [s^*_{m=1}, s^*_{m=0}]$, where $s^*_{m=1}$ is the closest element in $\mathbf{s}_{m=1}$ that is less than $\tau_{\mathbf{z}_0}$ and $s^*_{m=0}$ is the closest element in $\mathbf{s}_{m=0}$ that is greater than $\tau_{\mathbf{z}_0}$. Thus, we set the attack's calibrated loss threshold as the midpoint: $\tau_{\mathbf{z}_0} = \frac{1}{2}(s^*_{m=1} + s^*_{m=0})$. This sample-based loss threshold attack, wherein a different threshold is learned for each data point $\mathbf{z}_0$, is the attack we use for the CIFAR10 and Multi30k experiments.

A variation of this attack that we apply for the Purchase100 dataset is the global loss threshold, where $\tau_{\mathbf{z}_0} = \tau$ for every $\mathbf{z}_0$. In words, the same threshold value is applied when attacking the model on any data point. The procedure for threshold calibration is the same, except now $\mathbf{s}_{m=1}$ contains the losses for each of the data points each model was trained on and $\mathbf{s}_{m=0}$ contains the losses for the data points the models were not trained on.

\subsection{Evaluation procedure}

To evaluate the attack, we first randomly subsample a training dataset $\mc{S}$ from the full training dataset $\mc{D}$ and train a target model on $\mc{S}$. Denote by $\bar{\mc{S}}$ the data points in $\mc{D}$ that are not in $\mc{S}$.  We collect the losses $t(\mathbf{z}_0)$ of the target model on each data point $\mathbf{z}_0$ in $\mc{S}$ into a vector $\mathbf{t}_{m=1}$ and the losses of the target model on each data point in $\bar{\mc{S}}$ into a vector $\mathbf{t}_{m=0}$. The membership advantage for the target model is:
\begin{align}
    \text{Adv}_{\text{target}} = \tfrac{\left|\set{t(\mathbf{z}_0) \in \mathbf{t}_{m=1} \colon t(\mathbf{z}_0) < \tau_{\mathbf{z}_0}}\right|}{|\mc{S}|}
        - \tfrac{\left|\set{t(\mathbf{z}_0) \in \mathbf{t}_{m=0} \colon t(\mathbf{z}_0) < \tau_{\mathbf{z}_0}}\right|}{\bar{\mc{S}}}.
\end{align}
We repeat this evaluation procedure $n_{\text{target}}$ times, each time training a new target model on a newly sampled $\mc{S}$. The mean and standard deviation of the membership advantage over all experimental runs is what is reported in the paper figures.

Each shadow and target model is trained for $E$ epochs with checkpoints saved every $C$ epochs, where $E$ and $C$ differ per dataset. For the experiments in Section \ref{sec:params_tradeoff_nn} and Figure \ref{fig:MI_vs_params_nn}, the checkpoints for each shadow and target model that achieves the highest classification accuracy rate on the dataset's validation set is used for the experiment. The results in Figures \ref{fig:epochwise} and \ref{fig:tradeoff} are obtained for each checkpoint. For each curve in Figure \ref{fig:fixed_MI_val}, for all shadow and target models, we use the same number of epochs: the number of epochs (out of the checkpoints acquired) that achieves a membership advantage (Figure \ref{fig:fixed_MI}) or test error (\ref{fig:fixed_val}) closest to the one specified in the figure.

\subsection{Datasets and architectures}

We split each dataset into a ``full training dataset'' and a validation set. The full training dataset contains all the points on which membership inference will be performed. Each shadow and target model will be trained on a sample of the full training dataset such that the full training dataset would always contain both members (training points) and non-members (test points) for each model. The validation set is only used for calculating classification test error.

\textbf{Classification on Purchase100.} The Purchase100 dataset is based on Kaggle’s “acquire valued shoppers” challenge dataset subsequently processed by \cite{shokri2017membership}. It contains 197,324 length-600 binary feature vectors, each belonging to 1 of 100 classes. Each feature vector corresponds to a purchaser, and each entry of the vector corresponds to whether or not a particular product was purchased by the customer. The 100 classes correspond to purchasing styles. We use the first 180,000 data points as the full training dataset and the remaining data points for the validation set. We train two-layer neural networks with hidden dimension $w$, which we vary. We set $n_\text{shadow}=50$ and $n_\text{target}=50$. Each model is trained on a random sample of 10,000 data points. We use the ADAM optimizer \citep{kingma2014adam} with a learning rate of 0.001 for $E=3000$ epochs with checkpoints saved every $V=20$ epochs. For Figure \ref{fig:tradeoff}, we save checkpoints every $V=1$ epoch and only display the results for less than 3000 epochs for better visualization (each curve uses a different number of epochs, according to which provides best visualization).

\textbf{Image classification on CIFAR10.} The CIFAR10 dataset \citep{krizhevsky2009learning} contains 60,000 32$\times$32 RGB images, each belonging to 1 of 10 object classes. We use the 50,000 images in the official training dataset as our full training dataset, and the 10,000 images in the official validation dataset as our validation set. We train ResNet18 models \citep{he2016deep} to perform image classification on the dataset. To vary the models' widths, we follow \cite{nakkiran2021deep} and use convolutional layer widths (number of filters) of $[w, 2w, 4w, 8w]$ for different $w$ values. Note that $w=64$ yields the original ResNet18 architecture. We set $n_{\text{shadow}}=50$ and $n_{\text{target}}=50$, where each model is trained on a random sample of 25,000 images. We train for 50,000 gradient steps using the ADAM optimizer with a batch size of 128 (amounting to $\approx 256$ epochs through the training dataset), a learning rate of 0.0001 and the cross-entropy loss. Data augmentation is a common technique used in image classification, and so we also employ random translations of up to 4 pixels and random horizontal flipping during training, as was done by \cite{nakkiran2021deep}.

\textbf{Language translation on Multi30K.} The Multi30K dataset \citep{elliott2016multi30k} consists of 29,001 pairs of English-German sentences. We perform English to German translation on these sentences using the Transformer architecture \citep{vaswani2017attention}. To vary the models' widths, we follow \cite{nakkiran2021deep} and set the encoder/decoder feature sizes to $w$ and the fully connected layers' dimensions to $4w$ for different values of $w$. We train for 300 epochs using the ADAM optimizer with a learning rate of $0.0001$, a batch size of 128, and the cross-entropy loss over each token. We set $n_\text{shadow}=15$ and $n_\text{target}=15$ and train each model on a random sample of $14,500$ sentence pairs. In calculating the loss of a sentence pair for performing membership inference, we sum the cross-entropy loss values over all tokens in the sentence and divide by the sentence length.

\section{Additional Experiments}
\label{sec:additional-experiments}

\subsection{Blessing of Dimensionality for Multi30K}
In Figure \ref{fig:fixed_MI_val_multi30k}, we show the equivalent of Figure \ref{fig:fixed_MI_val} in the main paper for the transformer architecture on the Multi30k dataset. Similarly to the Purchase100 and CIFAR10 datasets, increasing the width of the neural network here improves either privacy (i.e.\ decreases membership advantage) or test accuracy when holding the other fixed via proper epoch tuning.

\begin{figure}[t]
	\centering
	\begin{subfigure}{0.4\textwidth}
    	\includegraphics[width=0.8\textwidth]{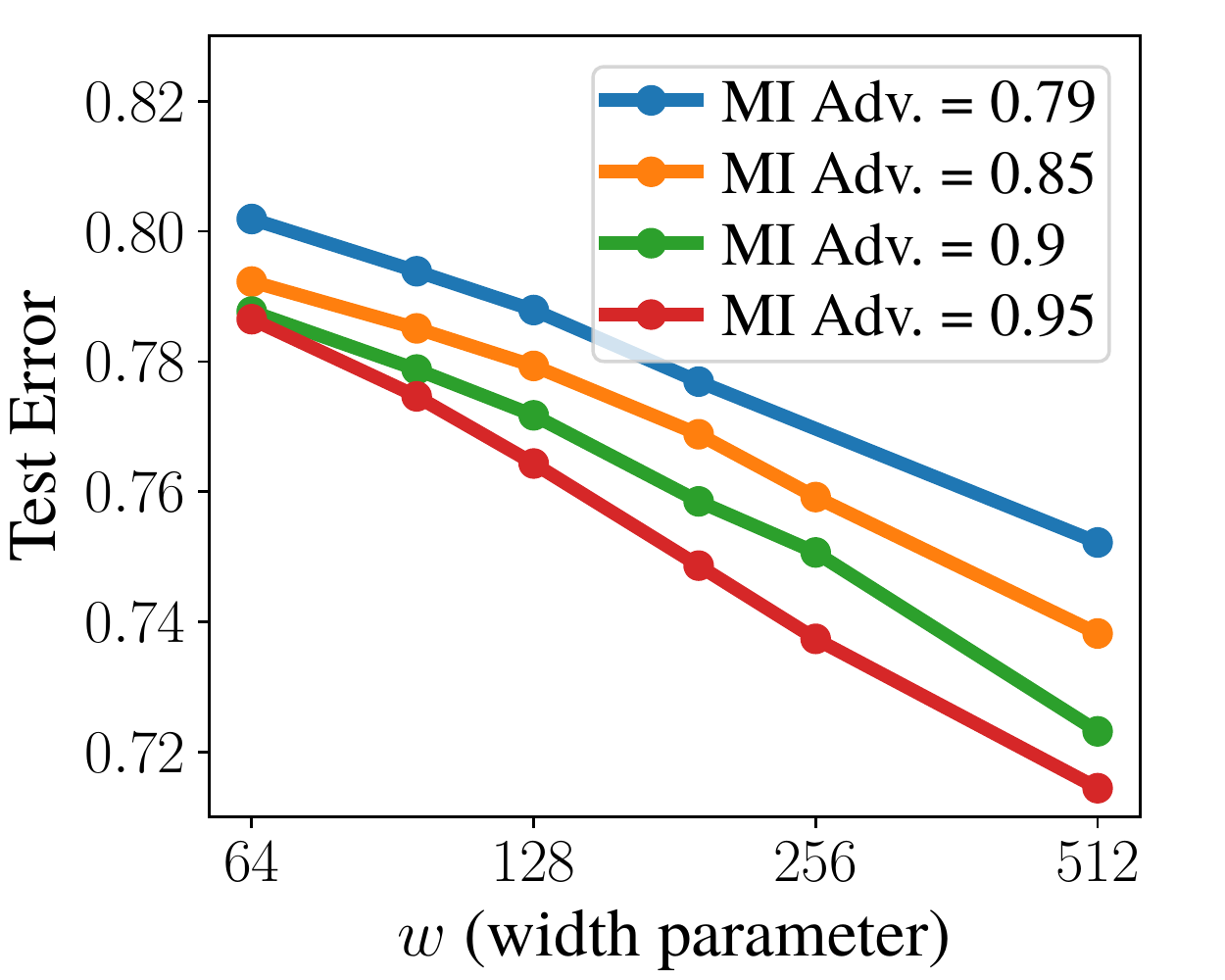}
    	\caption{Test error vs.\ net width for fixed MI adv.}
    	\label{fig:fixed_MI_multi30k}
	\end{subfigure}
	\begin{subfigure}{0.4\textwidth}
    	\includegraphics[width=0.8\textwidth]{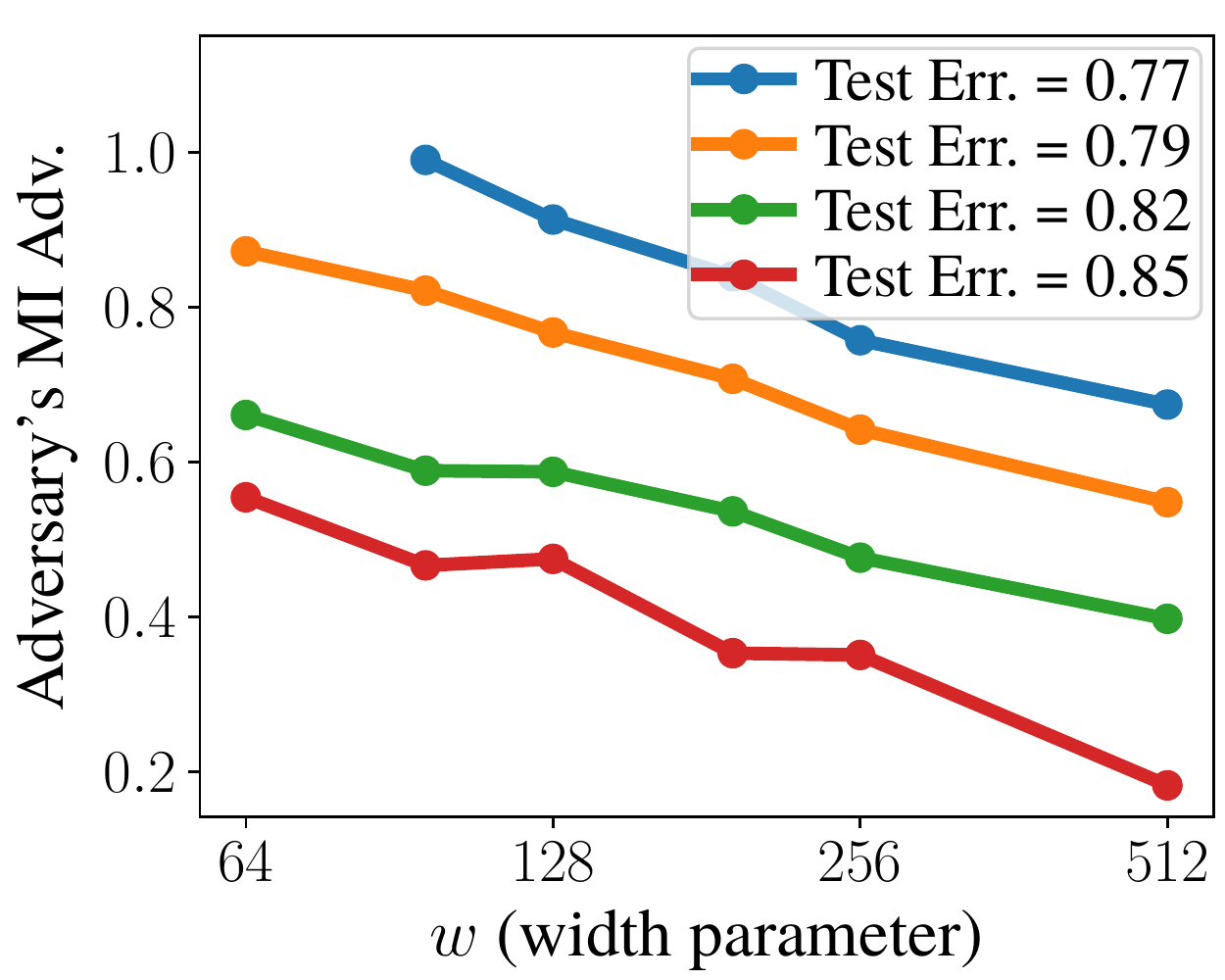}
    	\caption{MI adv.\ vs.\ net width for fixed test error.}
    	\label{fig:fixed_val_multi30k}
	\end{subfigure}
	\vspace{8pt}
	\caption{\textbf{Overparameterization with early stopping eliminates the privacy--utility trade-off on Multi30k.} This is similar to Figure~\ref{fig:fixed_MI_val} in the main body, but performed on the Multi30k dataset with the Transformer architecture. (a)~For each network width, we train the network until it reaches a given MI advantage value. We then plot the test error of the networks. Observe how test error decreases with parameters at a fixed MI advantage value. Thus, this eliminates the privacy--utility trade-off. Proper tuning of parameters and epochs together improves model accuracy without damaging its privacy. (b) Same as (a) but switching the roles of MI advantage and test error.}
	\label{fig:fixed_MI_val_multi30k}
\end{figure}

\begin{figure}[t]
	\centering
	\includegraphics[width=0.75\textwidth]{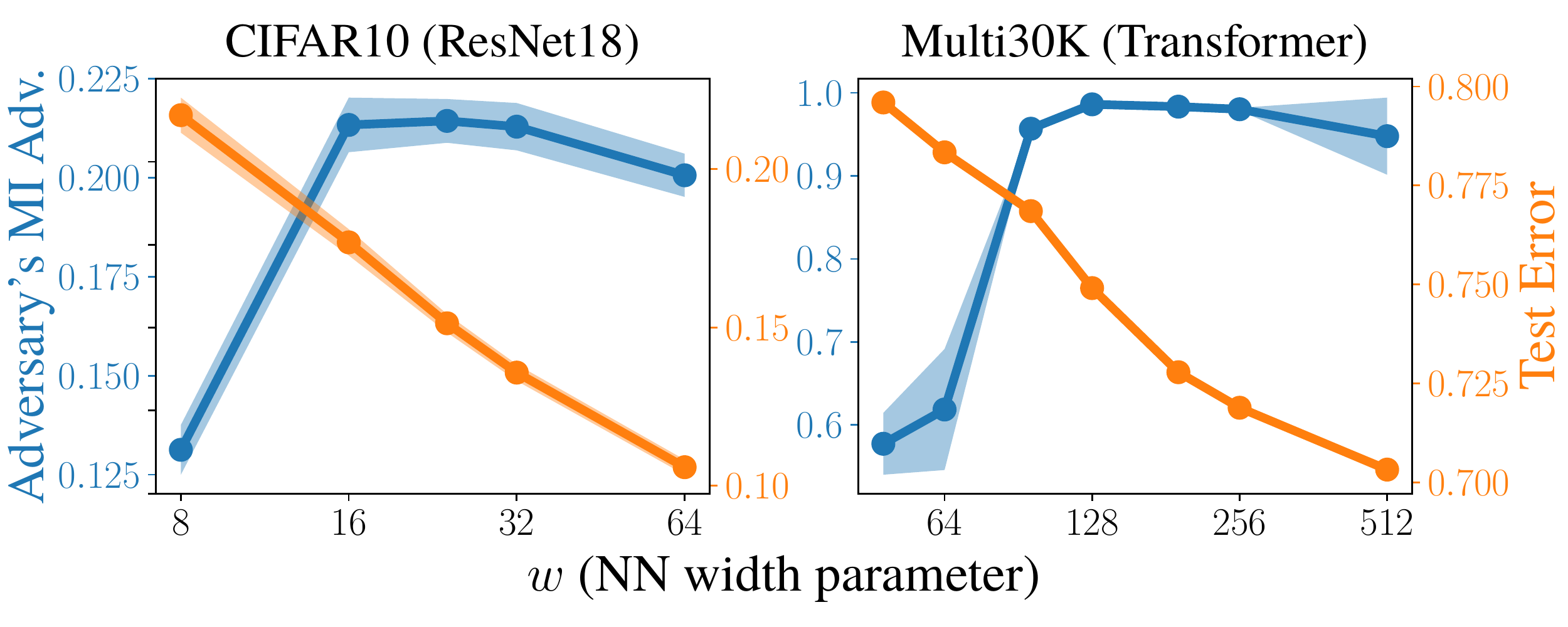}
	\caption{\textbf{Privacy vs.\ parameters (global loss threshold attack).} We repeat the experiment in Figure~\ref{fig:MI_vs_params_nn}, but now using the global (instead of sample-specific) loss threshold attack. Similarly, wider networks generally suffer from higher vulnerability to MI attacks while achieving lower test error.}
	\label{fig:MI_vs_params_nn_global}
\end{figure}

\begin{figure}[t!]
	\centering
	\includegraphics[width=0.75\textwidth]{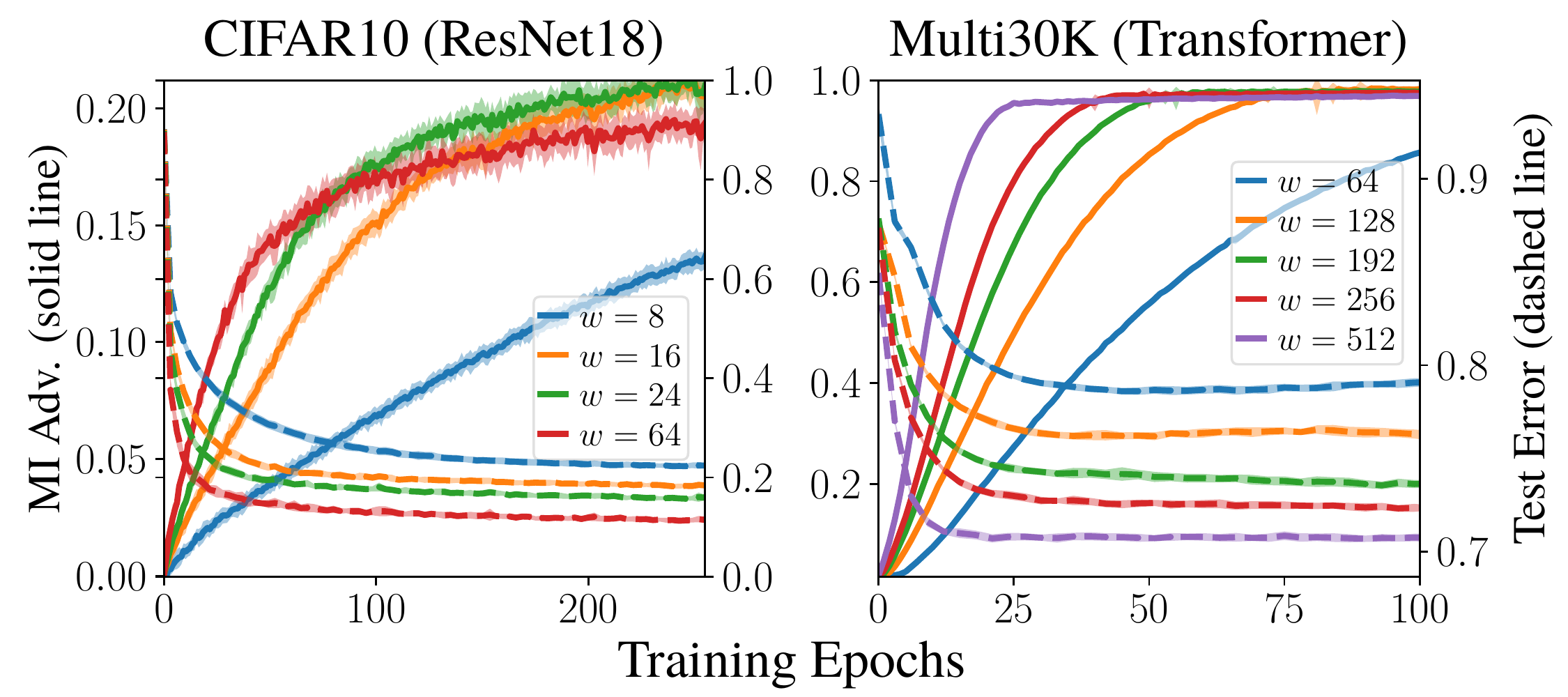}
	\caption{\textbf{Privacy vs.\ epochs (global loss threshold attack).} We repeat the experiment in Figure~\ref{fig:epochwise}, but now using the global (instead of sample-specific) loss threshold attack. Again, as epochs increase, membership advantage increases while test error decreases.}
	\label{fig:epochwise_global}
\end{figure}

\begin{figure}[t]
	\centering
	\includegraphics[width=0.7\textwidth]{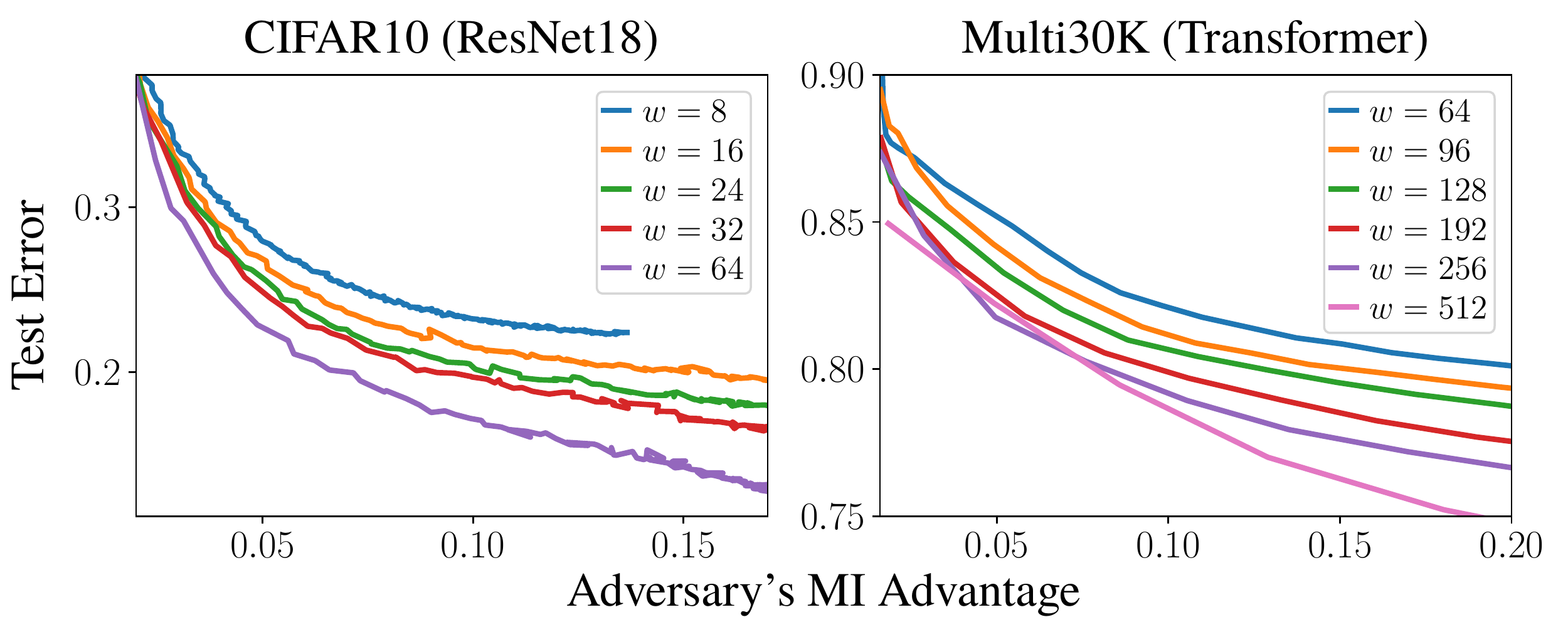}
	\caption{\textbf{Trade-offs (global loss threshold attack).} We repeat the experiment in Figure~\ref{fig:tradeoff}, but now using the global (instead of sample-specific) loss threshold attack. We observe again how wider networks are closer to the lower-left portion of the graph, indicating better privacy and better test accuracy compared to their narrower counterparts.}
	\label{fig:tradeoff_global}
\end{figure}

\begin{figure}[t]
	\begin{subfigure}{0.48\textwidth}
    	\includegraphics[width=\textwidth]{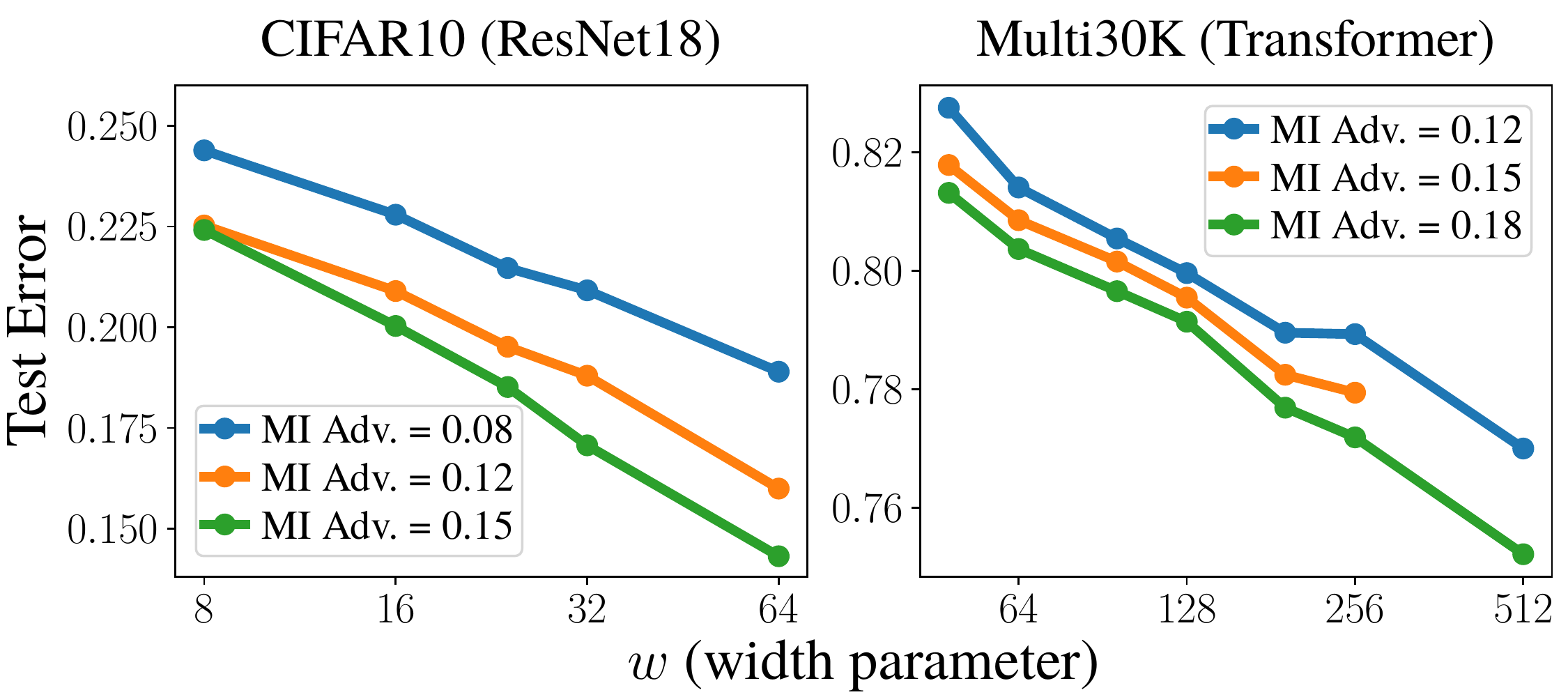}
    	\caption{Test error vs.\ net width for fixed MI adv.}
    	\label{fig:fixed_MI_global}
	\end{subfigure}
	\hspace*{\fill}
	\begin{subfigure}{0.48\textwidth}
    	\includegraphics[width=\textwidth]{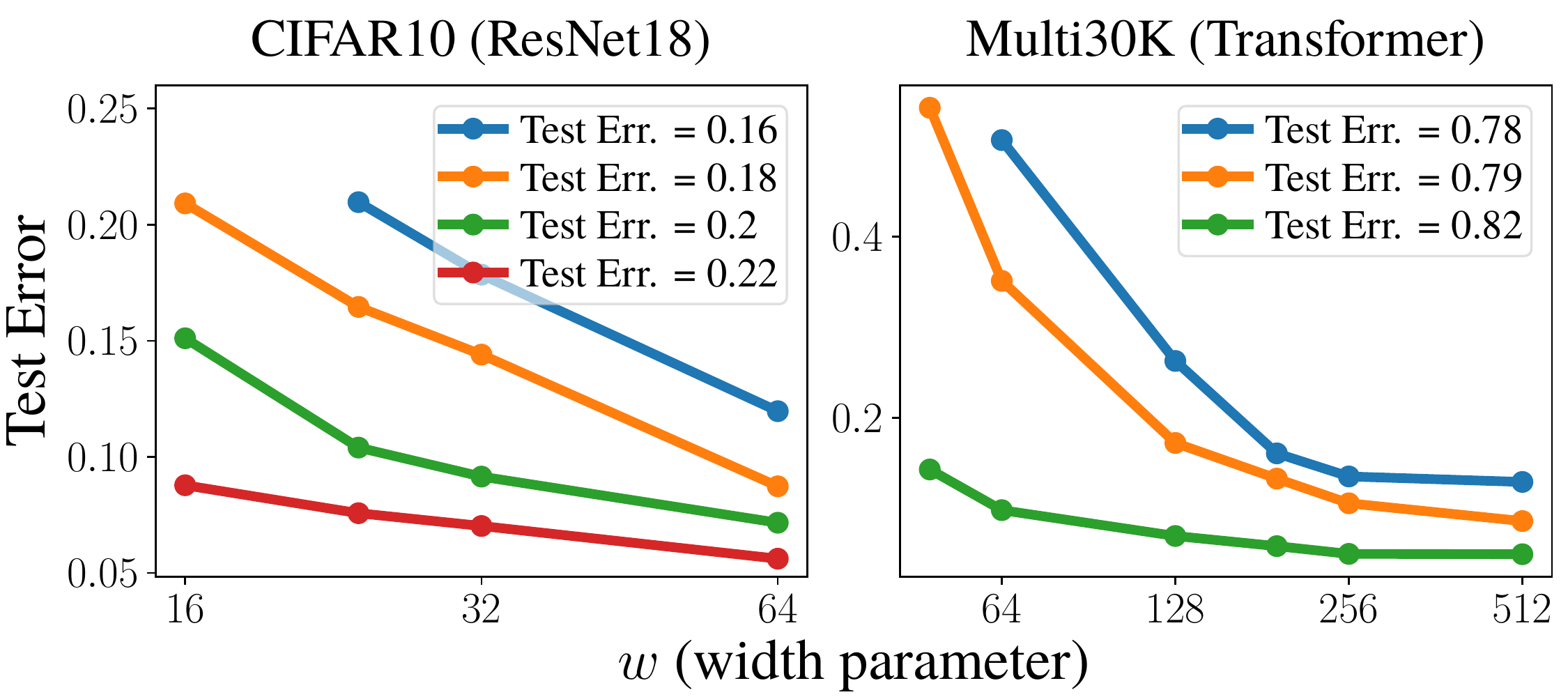}
    	\caption{MI adv.\ vs.\ net width for fixed test error.}
    	\label{fig:fixed_val_global}
	\end{subfigure}
	\vspace{8pt}
	\caption{\textbf{Overparameterization with early stopping eliminates the privacy--utility trade-off (global loss threshold).} Similar to Figures~\ref{fig:fixed_MI_val} and \ref{fig:fixed_MI_val_multi30k}, but using the global loss threshold. Increasing the parameters can improve either privacy or test accuracy when keeping the other fixed (by epoch tuning).}
	\label{fig:fixed_MI_val_global}
\end{figure}

\subsection{Global Loss Threshold Attack}

In Figures \ref{fig:MI_vs_params_nn}, \ref{fig:epochwise}, \ref{fig:tradeoff}, and \ref{fig:fixed_MI_val}, we used the sample-specific loss threshold attack for CIFAR10 and Multi30K, where a different loss threshold is learned for each data point. Here, we repeat the same experiments using the global loss threshold, where a single threshold value is used for all the data points. Note that in the mentioned figures, we already employed the global loss threshold attack for Purchase100. The trends we observe for the global loss threshold attack are similar to those of the sample-specific loss threshold attack. The results are shown in Figures \ref{fig:MI_vs_params_nn_global}, \ref{fig:epochwise_global}, \ref{fig:tradeoff_global}, and \ref{fig:fixed_MI_val_global}. We use $n_{\text{shadow}} = n_{\text{target}} = 15$ for both datasets in this experiment.

\subsection{Privacy-Utility Trade-offs for DP-SGD on CIFAR10}
We perform the same experiment of Figure \ref{fig:tradeoff} for CIFAR10 with ResNet18 models trained with DP-SGD \citep{abadi2016deep}. In DP-SGD, gradients are clipped to a maximum bound, and noise is added to the gradients before the gradient descent step. The model training procedure is guaranteed to be $(\epsilon, \delta)$ differentially private for some $\epsilon$ and $\delta$ according to the amount of noise added and the number of training epochs. The addition of noise also serves as a form of regularization. We thus obtain the regularization-wise privacy-utility trade-off for each network width by varying the amount of noise added. Specifically, we set the gradient clipping bound to 1, the number of epochs to 200, and $\delta$ to $\frac{1}{25000}$. For each $\epsilon \in \{1, 2, 3, ..., 14, 15, 16, 20, 50, 100\}$ and each learning rate in $\{0.1, 0.5, 1, 2, 4, 8\}$, we train 5 networks with noise added to the gradients such that the procedure is $(\epsilon, \delta)$ differentially private. Smaller $\epsilon$ parameters yield more noise, which serves as increased regularization. We try different learning rates as it has been observed that learning rate tuning can affect DP-SGD performance. We apply the global loss threshold attack and plot the mean test errors and mean membership advantage across the 5 networks for each $\epsilon$ and learning rate for different model widths in Figure~\ref{fig:dpsgd}. For each network width, we only include its Pareto optimal points. That is, we exclude a point if there exists another point that has both lower test error and lower membership advantage. We observe the same phenomenon as in Figure~\ref{fig:tradeoff}. Wider networks enjoy better privacy-utility trade-offs than narrower networks.

\subsection{TPR at FPR=1\%}

In Figure~\ref{fig:tpr_fpr_1percent}, we perform the same experiment as in Figure~\ref{fig:tradeoff} of the main paper, but we instead use the global loss threshold attack and report the maximum achievable true positive rate (TPR) when the false positive rate (FPR) is constrained to be at most 1\%. For the loss threshold attack, the adversary predicts the data point to be a member if the model's loss on the data point is below some $\tau$. When $\tau$ is increased, the adversary more frequently predicts the data point as being a member. This increases the adversary's TPR, but it will also increase its FPR. For the attack used in Figure~\ref{fig:tpr_fpr_1percent}, we choose the global thresholds for each individual network that maximizes the TPR under the constraint that the FPR is at most 1\%. We refer readers to \cite{carlini2021membership} for additional discussion on using the metric of TPRs for constrained FPRs. In this metric, we still observe the same blessing of dimensionality: wider networks can achieve lower test error and lower MI adversary TPRs than their narrower counterparts.

\begin{figure}[t]
	\centering
	\includegraphics[width=0.4\textwidth]{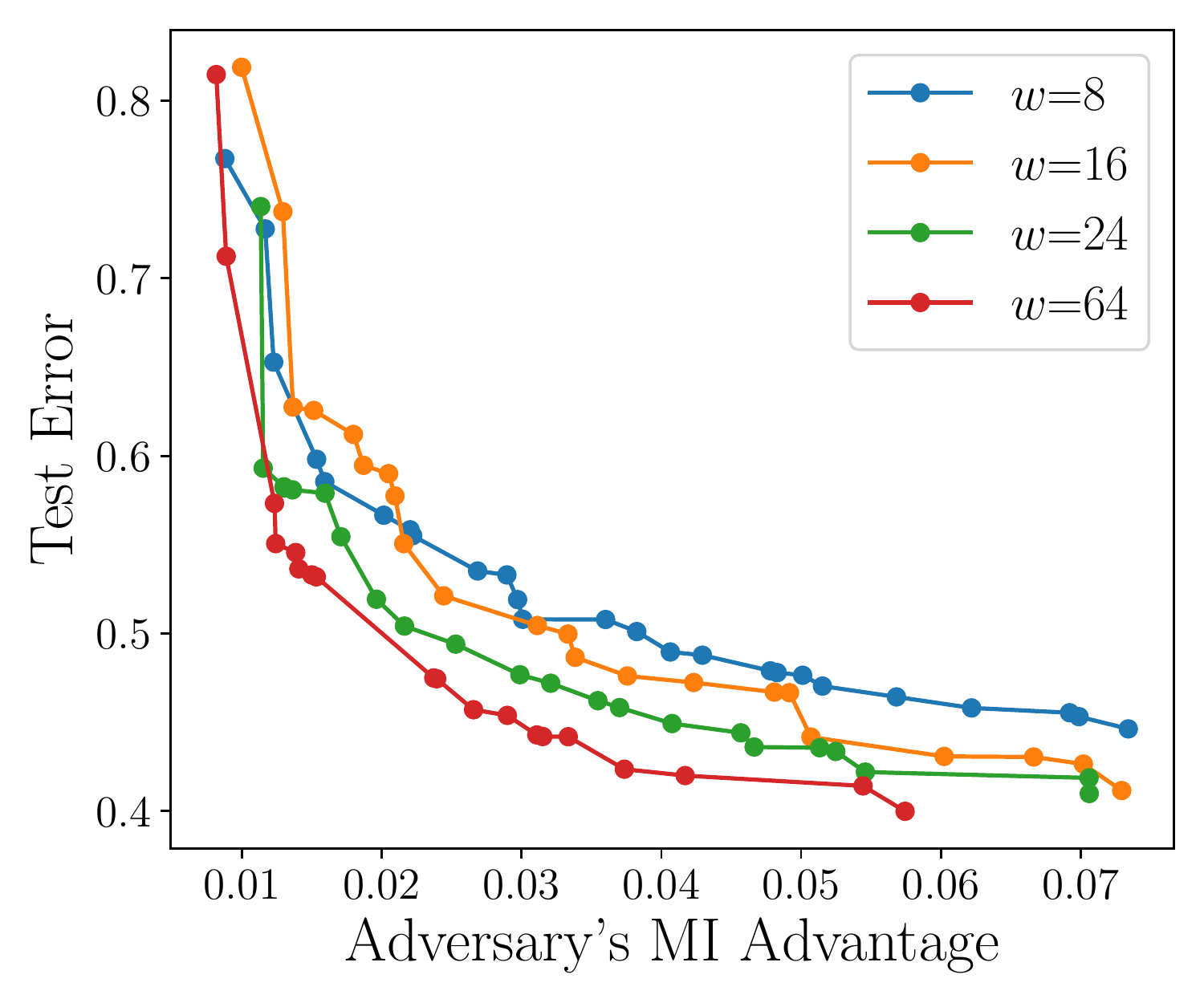}
	\caption{\textbf{DP-SGD Trade-off.} We train ResNet18 networks on CIFAR-10 with DP-SGD. We sweep through $\epsilon \in \{1, 2, 3, ..., 15, 16, 20, 50, 100\}$ and learning rates $\{0.1, 0.5, 1, 2, 4, 8\}$. For each $\epsilon$ and learning rate, we train 5 networks. Each point on the plot corresponds to the mean test error and mean MI advantage of the global loss threshold attack over the 5 networks for some $\epsilon$ and learning rate. We only include points that are Pareto optimal---we exclude a point if there exists another point with both lower test error and lower MI advantage. We fix the clipping bound to 1 and the number of epochs to 200. The plot shows that wider ResNet18 networks achieve better privacy--utility trade-offs than narrower networks when tuning the DP-SGD noise amount added.}
	\label{fig:dpsgd}
\end{figure}

\begin{figure}[t]
	\centering
	\includegraphics[width=0.8\textwidth]{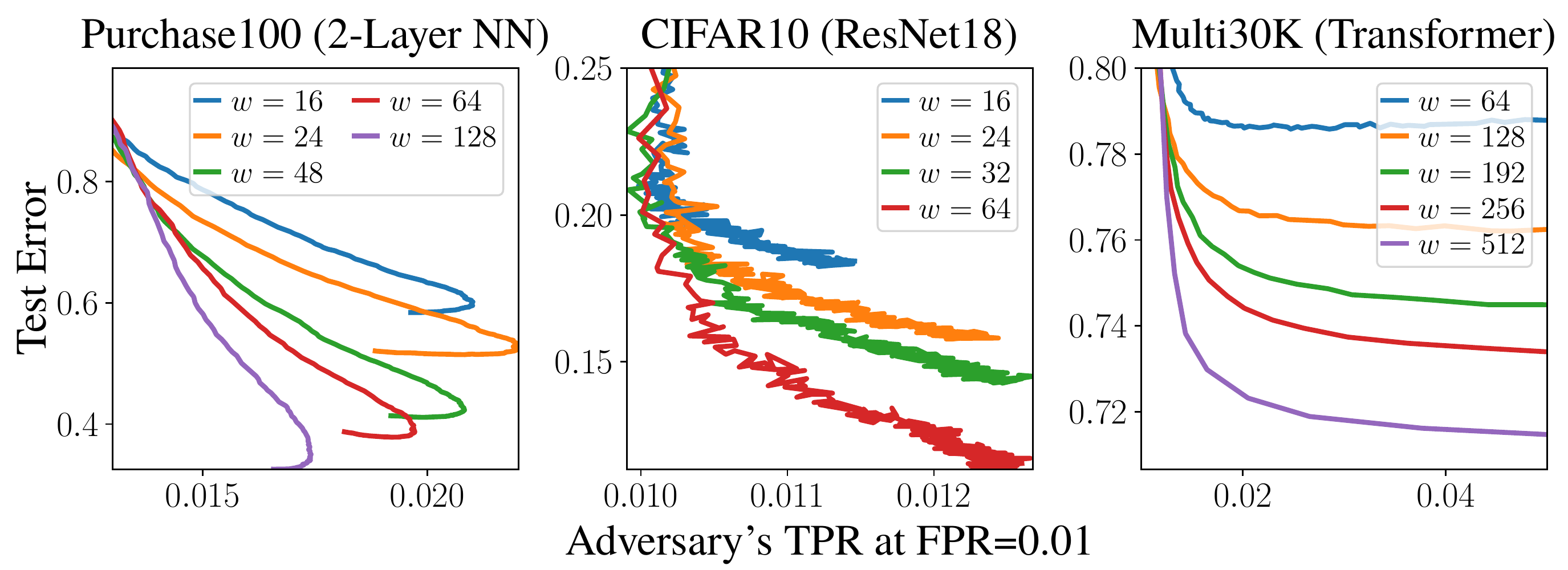}
	\caption{\textbf{TPR at FPR=1\%.} We show the privacy-utility trade-offs similar to Figure~\ref{fig:tradeoff} but reporting the global loss threshold's true positive rate (TPR) using the threshold value that maximizes TPR under the constraint that the false positive rate $\leq$ 0.01. We again observe wider networks enjoying better privacy-utility trade-offs than narrower ones.}
	\label{fig:tpr_fpr_1percent}
\end{figure}

\subsection{Parameter-Wise Privacy-Utility Trade-Off for Support Vector Machines}
\label{sec:svm} 
\begin{figure*}[t]
	\centering
	\includegraphics[width=\textwidth]{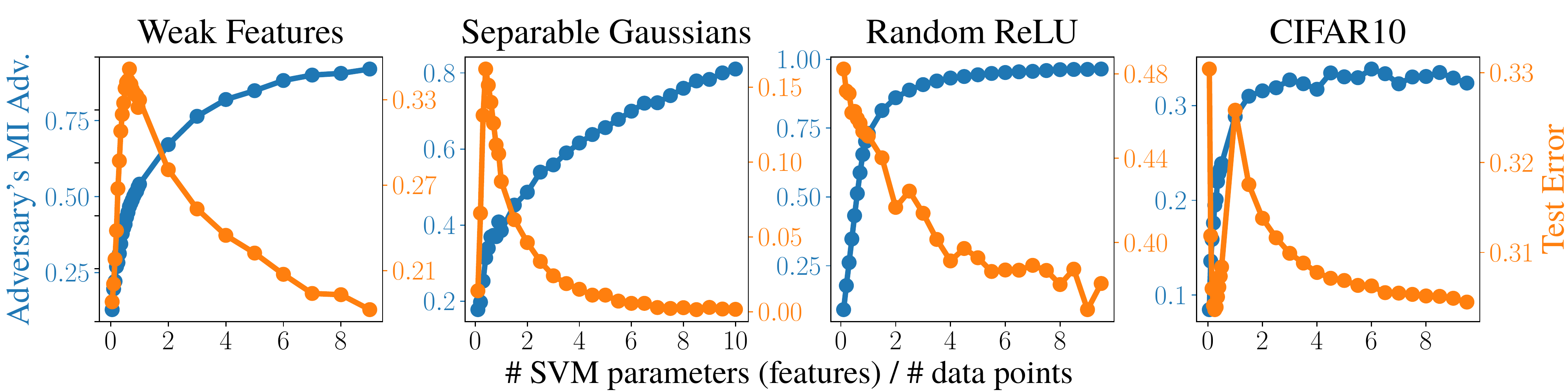}
	\caption{\textbf{Privacy vs.\ parameters (SVMs).} We demonstrate on SVMs for a variety of feature models how increasing overparameterization increases the adversary's MI advantage on the SVM model even as it decreases validation error. Thus, the number of parameters induces a privacy--utility trade-off.}
	\label{fig:MI_vs_params_svm}
\end{figure*}

We now consider support vector machines (SVMs), plotting both an adversary's membership advantage and the SVM model's validation error as a function of the number of parameters in Figure~\ref{fig:MI_vs_params_svm} for a variety of data models. We observe how MI increases (thus damaging privacy) while test error decreases (yielding a more accurate model) as the number of parameters grows. We consider data models based on those that have been shown to exhibit double descent in the overparameterized machine learning literature, including the weak features ensemble from \cite{muthukumar2021classification}, separable Gaussians with irrelevant features (based on synthetic dataset 1 of \citealp{belkin2018understand}), random ReLU features \citep{montanari2019generalization}, and random projections on two classes of CIFAR10. For the MI attack, we estimate the optimal LRT adversary \citep{tan2022parameters} by approximating the model output distributions as discrete histograms using Monte Carlo sampling over a minimum of 20,000 trials.

\subsubsection{SVM experimental setup}

This subsection provides details on the SVM experiments whose results are shown in Figure \ref{fig:MI_vs_params_svm}. For all SVM models, we use scikit-learn's SVC class \citep{pedregosa2011scikit}. When the number of SVM parameters is smaller than the number of data points in the training dataset, we add regularization $C=1$, where $C$ is the corresponding regularization parameter in scikit-learn's SVC class. Else, we use $C=10^{20}$, essentially applying no regularization to yield the hard-margin SVM. The hard-margin SVM has been studied considerably in the double descent literature \citep{montanari2019generalization, muthukumar2021classification}, especially with regards to its relationship to logistic regression trained with gradient descent \citep{soudry2018implicit, ji2019implicit}.

We use the optimal MI adversary \citep{tan2022parameters}, which is a likelihood ratio test, as our MI attack. Suppose we are given two discrete distributions over values $s_i$ with probability mass functions $q_{m=0}$ and $q_{m=1}$. The optimal adversary $A^*$ is defined by:
\begin{align}
    A^*(s_i) = \begin{cases}
    1 &\mbox{if } q_{m=1}(s_i) > q_{m=0}(s_i),\\
    0 &\mbox{otherwise}
    \end{cases}.
\end{align}

For each experiment, the general procedure is as follows. We first generate a $D$-dimensional data point $\mathbf{x}_0$ with binary label $y_0$, for some $D$. This is the data point on which MI will be performed. Then, for an integer $p$, we perform the following procedure $L$ times. We first generate an $n \times p$ training dataset matrix $\mathbf{X}$, for some $n$, and a corresponding label vector $\mathbf{y}$. Generally, these are distributed in the same way as $(\mathbf{x}_0, y_0)$. All experiments here are binary classification tasks, so $\mathbf{y}_i \in \{-1, +1\}$ for $i \in \{1, 2, ..., n\}$. We then apply label noise to $\mathbf{y}$: we flip each label $\mathbf{y}_i$ to the other class with probability $\alpha$. Afterwards, we learn an SVM on $\mathbf{X}$ and $\mathbf{y}$. We then denote by $\hat{y}_0$ the signed distance of $\mathbf{x}_0$ to the decision hyperplane of the learned SVM. We collect the $\hat{y}_0$ of all $L$ learned SVMs into an output vector $\mathbf{\hat{y}}_{m=0}$. We then repeat the same procedure another $L$ times, but this time, before learning the SVM on $\mathbf{X}$ and $\mathbf{y}$, we first replace the first rows $\mathbf{X}_1 = \mathbf{x}_0$ and $\mathbf{y}_1 = y_0$. Label noise is never applied to $y_0$. We collect the resulting $L$ signed distances to the learned SVM hyperplanes into the output vector $\mathbf{\hat{y}}_{m=1}$. We form discrete histograms for both $\mathbf{\hat{y}}_{m=0}$ and $\mathbf{\hat{y}}_{m=1}$ with bin width $b$. Finally, we perform the optimal adversary attack on these histograms and measure the corresponding membership advantage. This entire experiment is repeated for multiple values $p$ to generate Figure \ref{fig:MI_vs_params_svm}.

The following subsections provide the distributions of $\mathbf{x}_0$, $y_0$, $\mathbf{X}$, and $\mathbf{y}$, as well as the hyperparameters $n$ (number of data points), $D$ (full data dimensionality), label noise probability $\alpha$, $L$ (number of samples used to form the histogram), histogram bin width $b$, and the set of number of features $p$ investigated for each data model.

\subsubsection{Weak features}

The weak features experiment is based on the weak features ensemble discussed in Definition 9 of \cite{muthukumar2021classification}. In our experiment, we let $D=1000$, $\mathbf{x}_0 \sim \mathcal{N}(\mathbf{1}_D, \mathbf{I}_D)$, and $y_0 = 1$. We perform the experiment for $p \in \{5, 10, 15, ..., 95, 100, 200, 300, ..., 900\}$ with number of data points $n=100$, number of samples $L = 20,000$, histogram bin width $b=0.05$, and label noise probability $\alpha=0.2$. The $n \times p$ training dataset matrix $\mathbf{X}$ is generated such that the i'th row $\mathbf{X}_i \sim \mathcal{N}(z_i, \mathbf{I}_p)$, with $z_i \sim \mathcal{N}(0, 1)$. The elements of the label vector $\mathbf{y}$ are defined by $y_i = \text{sign}(z_i)$. In essence, each element of a training data point $\mathbf{X}_i$ is the true signal $z_i$ (on which the label $y_i$ is based) corrupted by Gaussian noise.

\subsubsection{Separable Gaussians}

The separable Gaussians model is based on synthetic dataset 1 of \cite{belkin2018understand} with some modifications. In our experiment, we set $D=1000$ and generate $\mathbf{x}_0$ by sampling its individual elements as:
\begin{align}
    \mathbf{x}_{0,j} \sim \begin{cases}
    \mathcal{N}(1, 1) &\mbox{if } j \leq 100 \\
    \mathcal{N}(0, 1) &\mbox{otherwise}
    \end{cases},
\end{align}
with true label $y_0 = 1$. We perform the experiment for $p \in \{10, 20, 30, ..., 90, 100, 150, 200, 250, ..., 1000\}$ with number of data points $n=100$, number of samples $L = 10,000$, and histogram bin width $b=0.05$.. Each element of the label vector $\mathbf{y}_i$ is randomly selected from $\{-1, +1\}$ with uniform probability. The individual elements (row $i$ and column $j$) of the training dataset matrix $\mathbf{X}$ are distributed as: 
\begin{align}
    \mathbf{X}_{i, j} \sim \begin{cases}
    \mathcal{N}(\mathbf{y}_i, 1) &\mbox{if } j \leq 100 \\
    \mathcal{N}(0, 1) &\mbox{otherwise}
    \end{cases}.
\end{align}
Label noise with probability $\alpha = 1$ is then applied to $\mathbf{y}$ after $\mathbf{X}$ is generated. Essentially, the first $\min(100, p)$ features of each data point depend on its true class, and the remaining features are irrelevant (independent of the class). Thus, in the overparameterized regime, as $p$ increases, we are including more irrelevant features to the model.

\subsubsection{Random ReLU features}

The random ReLU features model has been studied by multiple papers, such as \cite{rahimi2007random} and \cite{montanari2019generalization} (section 3). In essence, it is a two-layer ReLU neural network with fixed first-layer random weights. Different from the previous SVM data models, here $\mathbf{x}_0$ is defined differently for each trained SVM model because of the random projections. Instead, there is a latent data vector $\mathbf{z}_0$ that is kept fixed for all experiments and on which MI is performed. We set $D=200$, and generate $\mathbf{z}_0$ by sampling it from $\mathcal{N}(0, \mathbf{I}_D)$. We perform the experiment for $p \in \{10, 20, 30, ..., 90, 100, 150, 200, ..., 950\}$ with number of data points $n=100$, number of samples $L = 100,000$, histogram bin width $b=0.001$, and no label noise. To generate the training data, a random $p \times D$ ``featurizer'' matrix $\mathbf{W}$ is first generated by sampling each row independently from the $D$-dimensional unit sphere. Then, an $n \times D$ feature data matrix $\mathbf{Z}$ is generated by sampling each element iid standard normal. The training data matrix $\mathbf{X} = \max(0, \mathbf{Z}\mathbf{W}^T)$, where the $\max$ operation is applied elementwise. The MI data point $\mathbf{x}_0$ is defined as $\mathbf{x}_0 = \max(0, \mathbf{z}_0^\top \mathbf{W}^\top)$. Note that since $\mathbf{W}$ is sampled for each trained SVM, $\mathbf{x}_0$ changes for each experimental run. To generate the labels of the data points, first, a random vector $\beta$ is sampled uniformly from the $D$-dimensional sphere of radius 4 (such that $||\beta||_2 = 4$). Then, $y_i$ is assigned class $1$ with probability $\frac{1}{1 + e^{-\mathbf{Z}\beta}}$ and class $-1$ otherwise. The label $y_0$ of $x_0$ is defined similarly and is assigned class $1$ with probability $\frac{1}{1 + e^{-\mathbf{Z}\beta}}$ and class $-1$ otherwise. Essentially, the class of a data point depends only on $\mathbf{Z}$, and the training set consists of random projections of $\mathbf{Z}$ that are then passed through the ReLU operation.

\subsubsection{CIFAR10}

To experiment on real data, we train SVMs on random projections of a subset of the CIFAR10 dataset \citep{krizhevsky2009learning}. We first define $\mathbf{z}_0$ to be the first image of the training dataset with class ``airplane'' converted to grayscale and then vectorized. We perform the experiment for $p \in \{10, 20, 30, ..., 90, 100, 200, 300, ..., 1800, 1900\}$ with number of data points $n = 200$, number of samples $L = 10,000$, histogram bin width $b=0.05$, and no label noise. To generate the $n \times p$ data matrix, we first randomly sample $\frac{n}{2}$ images uniformly from the ``airplane'' images of the dataset (excluding $\mathbf{x}_0$) and $\frac{n}{2}$ images from the ``automobile'' images of the dataset. We convert the images to grayscale, vectorize them, and collect them into a matrix $\mathbf{Z}$ (where each row is a vectorized image). Since each image is of size $32 \times 32$, the vectorized image is $D=1024$ dimensional. We then sample a $p \times 1024$ random projections matrix $\mathbf{W}$, where each row is sampled uniformly from the $1024$-dimensional unit sphere. Finally, the data matrix $\mathbf{X} = \mathbf{Z}\mathbf{W}^\top$. The MI point $\mathbf{x}_0 = \mathbf{z}_0^\top \mathbf{W}^\top$. The labels of each data point is $-1$ if it originated from an ``airplane'' image and $+1$ if it originated from an ``automobile'' image.

\vfill

\end{document}